\documentclass[twoside,11pt]{article}

\usepackage{epsfig}
\usepackage{amssymb}
\usepackage{natbib}
\usepackage{hyperref}
\usepackage{graphicx}
\usepackage{thumbpdf}
\usepackage{amsfonts,amstext,amsmath,amsthm}
\usepackage{accents}
\usepackage{color}
\usepackage{rotating}
\usepackage[utf8]{inputenc}
\usepackage{booktabs}
\usepackage{tikz}
\usepackage{multirow}
\usepackage{float}
\usepackage[labelfont=bf, width=15cm, font=footnotesize]{caption}
\usepackage[onehalfspacing]{setspace}
\usepackage[left=3cm, right=3cm, top=2.5cm, bottom=3cm]{geometry}
\usepackage{pdflscape}
\setcitestyle{authoryear,open={(},close={)}}
\newcommand*{\doi}[1]{\href{http://dx.doi.org/#1}{doi: #1}}

\makeatletter
\def\blfootnote{\gdef\@thefnmark{}\@footnotetext}
\makeatother

\usepackage{tikz}

\usetikzlibrary{positioning, calc, shapes.geometric, shapes.multipart,
  shapes, arrows.meta, arrows, patterns,
  decorations.markings, external, trees}
\tikzstyle{line} = [draw, -latex']
\tikzstyle{Arrow} = [
        thick,
        decoration={
                markings,
                mark=at position 1 with {
                        \arrow[thick]{latex}
} },
        shorten >= 3pt, preaction = {decorate}
        ]

\usepackage{accents}

\newcommand{\rZ}{Z}
\newcommand{\rY}{Y}
\newcommand{\rX}{\mX}

\newcommand{\rz}{z}
\newcommand{\ry}{y}
\newcommand{\rx}{\xvec}

\newcommand{\pZ}{F_\rZ}
\newcommand{\pY}{F_\rY}

\newcommand{\dZ}{f_\rZ}

\newcommand{\g}{g}
\newcommand{\h}{h}

\newcommand{\hY}{h_\rY}

\newcommand{\basisy}{\avec}

\newcommand{\parm}{\varthetavec}
\newcommand{\eparm}{\vartheta}

\newcommand{\shiftparm}{\betavec}

\newcommand{\ie}{\textit{i.e.,}~}
\newcommand{\eg}{\textit{e.g.,}~}
\newcommand{\cf}{\textit{cf.}~}

\newcommand{\Prob}{\mathbb{P}}
\newcommand{\Ex}{\mathbb{E}}
\newcommand{\RR}{\mathbb{R}}

\usepackage{dsfont}
\newcommand{\I}{\mathds{1}}

 \DeclareMathOperator{\logit}{logit}
 \DeclareMathOperator{\expit}{expit}

 \DeclareMathOperator*{\argmax}{{arg\,max}}

 \DeclareMathOperator{\ND}{N}

\def \avec {\text{\boldmath$a$}}

\def \xvec {\text{\boldmath$x$}}

\def \rX {\text{\boldmath$X$}}

 \def \calD {\mathcal D}
 
 \def \calF {\mathcal F}
 \def \calG {\mathcal G}

 \def \calP {\mathcal P}

 \def \calY {\mathcal Y}

\def \betavec         {\text{\boldmath$\beta$}}

\def \varthetavec     {\text{\boldmath$\vartheta$}}

\newcommand{\ubar}[1]{\underaccent{\bar}{#1}}

\newcommand{\pkg}[1]{\textbf{#1}}

\newcommand{\proglang}[1]{\textsf{#1}}

\newcommand{\B}{\mathsf{B}}
\newcommand{\dd}{\,\mathrm{d}}
\newcommand{\given}{\,\vert\,}
\newcommand{\vs}{\emph{vs.}~}

\newcommand{\Be}{M}
\newcommand{\be}{m}
\newcommand{\pmem}{F}
\newcommand{\dmem}{f}
\newcommand{\pens}{\bar \pmem_\Be}
\newcommand{\pensi}{\bar \pmem_{\Be,i}}
\newcommand{\pmemb}{\pmem_\be}
\newcommand{\dens}{\bar \dmem_\Be}
\newcommand{\dmemb}{\dmem_\be}
\newcommand{\hb}{\h_\be}
\newcommand{\hens}{\bar \h_\Be}
\newcommand{\ensparm}{\bar\parm_\Be}
\newcommand{\ensshiftparm}{\bar\shiftparm_\Be}
\newcommand{\enscomplex}{\bar\eta_\Be}

\newcommand{\ls}{LS}
\newcommand{\cs}{CS}
\newcommand{\si}{SI}
\newcommand{\ci}{CI}
\newcommand{\lsx}{\ls$_\rx$}

\newcommand{\cib}{\ci$_\B$}
\newcommand{\csb}{\cs$_\B$}
\newcommand{\silsx}{\si-\lsx}
\newcommand{\sicsb}{\si-\csb}
\newcommand{\sicsblsx}{\si-\csb-\lsx}
\newcommand{\ciblsx}{\cib-\lsx}

\DeclareMathOperator{\NLL}{NLL}
\DeclareMathOperator{\RPS}{RPS}
\DeclareMathOperator{\BS}{BS}

\theoremstyle{definition}
\newtheorem{proposition}{Proposition}
\newtheorem*{proposition*}{Proposition}
\newtheorem{definition}{Definition}

\newtheorem{remark}{Remark}
\newtheorem{example}{Example}
\newtheorem{theorem}{Theorem}

\newtheorem{corollary}{Corollary}
\newtheorem*{corollary*}{Corollary}

\title{\bf Deep Interpretable Ensembles%
\blfootnote{Corresponding authors: \texttt{lucasheinrich.kook@uzh.ch, sick@zhaw.ch}%
\newline
Preprint. Version May 25, 2022. Licensed under CC-BY.}} 

\author{%
Lucas Kook\textsuperscript{1,2}, 
Andrea G\"otschi\textsuperscript{1},
Philipp F. M. Baumann\textsuperscript{3},\\
Torsten Hothorn\textsuperscript{1},
Beate Sick\textsuperscript{1,2}
}

\date{%
\footnotesize
\textsuperscript{1}Epidemiology, Biostatistics \& Prevention Institute, 
University of Zurich, CH-8001\\
\textsuperscript{2}Institute for Data Analysis and Process Design,
Zurich University of Applied Sciences, CH-8400\\
\textsuperscript{3}KOF Swiss Economic Institute, ETH Zurich, CH-8092\\
} 

\begin{document}

\maketitle

\begin{abstract}%

Ensembles improve prediction performance and allow uncertainty quantification by
aggregating predictions from multiple models.
In deep ensembling, the individual models are usually black box neural networks,
or recently, partially interpretable semi-structured deep transformation models.
However, interpretability of the ensemble members is generally lost upon
aggregation. This is a crucial drawback of deep ensembles in high-stake
decision fields, in which interpretable models are desired.
We propose a novel transformation ensemble which aggregates probabilistic 
predictions with the guarantee to preserve interpretability and yield uniformly
better predictions than the ensemble members on average.
Transformation ensembles are tailored towards interpretable deep transformation
models but are applicable to a wider range of probabilistic neural networks.
In experiments on several publicly available data sets, we demonstrate that 
transformation ensembles perform on par with classical deep ensembles in terms of
prediction performance, discrimination, and calibration.
In addition, we demonstrate how transformation ensembles quantify both aleatoric
and epistemic uncertainty, and produce minimax optimal predictions under certain 
conditions.

\end{abstract}

\section{Introduction} \label{sec:intro}

The need for interpretable yet flexible and well-performing prediction models is
great in high-stakes decisions fields, such as medicine. Practitioners need 
to be able to understand how a model arrives at its predictions, and how confident 
these predictions are, to assess a model's trustworthiness.
For this purpose, \cite{rudin2018stop} proposes the use of intrinsically interpretable
models. A model is deemed intrinsically interpretable, if it possesses a transparent 
structure, such as sparsity or additivity with parameters that can be interpreted,
\eg as log odds-ratios or log hazard-ratios.
For instance, traditional statistical regression models, \eg linear regression, or the 
Cox proportional hazards model \citep{cox1972regression}, and also decision trees fall in 
this category. However, contemporary applications may involve more complex data or require
methods, for which intrinsic interpretability can only partly be achieved. For models
that are not intrinsically interpretable there exist \emph{post hoc} or model-agnostic 
explanation methods \citep[for an overview, see][]{molnar2020interpretable}.

Nowadays the available data for prediction is often a mix of structured tabular and
unstructured data, such as images, text, or speech. The prediction target can be 
continuous (\eg time-to-event), or discrete (\eg number of recurrences, or an ordinal 
score, such as tumour-grading). Hence, the ideal prediction model should in summary 
(i) be as interpretable as possible, (ii) yield accurate and calibrated probabilistic
predictions, and (iii) be applicable to multi-modal input data and different 
kinds of outcomes. One group of models fulfilling these requirements are deep 
transformation models \citep{sick2020deep,baumann2020deep,kook2020ordinal,rugamer2021timeseries}.
Deep transformation models estimate the conditional distribution of an outcome
given all available data modalities. However, a single instance of such a model may 
suffer from overfitting and overconfident predictions, which may lead to severe 
prediction errors when applied to novel data. 

To build well-performing models, it is common practice to aggregate the predictions 
of several models, mitigating overconfident and potentially inaccurate predictions 
\citep{buhlmann2012bagging}. Aggregating several models is referred to as ensembling,
and the resulting final model is called an ensemble. 
Deep ensembling, \ie aggregating predictions of (often as few as 3--5) deep neural
network that are fitted on the same data but with different random initializations,
has been demonstrated to notably improve prediction performance and uncertainty 
quantification \citep{lakshminarayanan2016deepensembles}.
We refer to this approach as classical deep ensembling. However, classical deep 
ensembles are black boxes even if their members are somewhat interpretable 
individually.

In this paper, we propose transformation ensembles, which preserve structure and
interpretability of their members. In addition, transformation ensembles improve
predictions and allow uncertainty quantification with distributional deep neural 
networks. 

\subsection{Our contribution}
\label{subsec:contrib}

We present transformation ensembles as a novel way to aggregate predicted
cumulative distribution functions (CDFs) derived from deep neural networks with
a special focus on deep transformation models for (potentially) semi-structured
data. We show that transformation ensembles (i) preserve structure and
interpretability of their members, (ii) improve the prediction performance akin
to classical deep ensembles, and (iii)
minimize worst-case prediction error in special cases. We showcase these
properties theoretically and demonstrate the benefits of transformation
ensembles empirically on several semi-structured data sets. With transformation
ensembles we are able to provide empirical evidence for answering open questions
in deep ensembling \citep{abe2022deep}. For instance, the increased flexibility
of classical deep ensembles over their members does not seem to be necessary for 
improving prediction performance or allowing uncertainty quantification.

\paragraph{Structure of this paper}
We discuss how transformation ensembles relate to classical deep ensembles and other
ensembling strategies from the forecasting literature in Section~\ref{subsec:related}.
We give the necessary background on scoring rules, ensembling and transformation models 
in Section~\ref{sec:bg}. Afterwards, we present transformation ensembles 
(Section~\ref{sec:trafoensemble}) and demonstrate how they improve prediction performance 
similar to deep ensembles while preserving interpretability (Section~\ref{sec:results}).
We end with a discussion of our results and potential directions of future research
(Section~\ref{sec:outlook}). Mathematical notation used in this article is summarized
in Appendix~\ref{app:notation}. In further appendices, we provide computational details 
(Appendix~\ref{app:experiments}), and additional theoretical (Appendix~\ref{app:theory})
and empirical (Appendix~\ref{app:figures}) results.

\subsection{Related work}
\label{subsec:related}

We recap ensembling techniques from machine learning, deep learning and ensembles
with a specific focus on proper scoring rules. Note that we focus on probabilistic
prediction in this article, but not point prediction. For completeness we mention 
some ensembles for point prediction, \ie classification and conditional means, below.

\paragraph{Ensembling in machine learning}
Ensembling methods have been used for some decades now in statistics and machine
learning \citep[see][for an overview]{buhlmann2012bagging}. For instance,
\citet{hansen1990neural} aggregated classifications from neural networks with different
initializations by majority vote and \citet{breiman1996bagging} proposed to aggregate 
models or predictions obtained from bootstrap samples (bagging).
In practice, decision trees as base models have been vastly successful and the random 
forest \citep{breiman2001random} is probably the most well-known ensemble algorithm.
The main goal of these ensemble methods is to improve upon the individual members'
prediction performance in terms of both calibration and sharpness, rather than to 
obtain interpretable models \citep[\eg][]{breiman2001statistical}.

\paragraph{Classical deep ensembles}
Classical deep ensembles combine predictions of several deep neural networks, by
training several random instances of the same model on the same data
\citep{lakshminarayanan2016deepensembles}.
In contrast to the bagging algorithms discussed above, heterogeneity of deep ensemble
members does not stem from bootstrapping the data, but rather from stochasticity 
in initializing and optimizing the neural networks.
Several contributions suggest that deep ensembles benefit prediction performance, 
uncertainty quantification and out-of-distribution generalization 
\citep{fort2019deep,wilson2020bayesian,hoffmann2021deep}. 

\citet{abe2022deep} question the extent to which these benefits hold.
For instance, the authors suggest that more complex models may show a gain in
prediction performance similar to classical ensembles. 
Other work \citep{sluijterman2022confident} investigates ensembling conditional 
mean functions instead of probabilistic predictions. However, we do not further address
conditional mean ensembles in this article.

\paragraph{Quasi-arithmetic pooling with proper scoring rules and minimax optimality}
Proper scoring rules \citep{gneiting2007strictly} are a natural choice to evaluate
probabilistic forecasts (see Section~\ref{subsec:scoring} for more detail).
For nominal outcomes, \citet{neyman2021proper} coin the term \emph{quasi-arithmetic pooling},
for aggregating predicted densities, $p_1, \dots, p_\Be$ from $\Be$ models by
\begin{align} \label{eq:qa}
   p_\Be^g = g^{-1} \left(\sum_\be w_\be g \circ p_\be\right),
\end{align}
where $g$ is any continuous, non-decreasing function and $w_\be$ are non-negative weights
summing to one. The authors prove that for nominal outcomes, certain combinations of scoring
rules and ensembling methods are minimax optimal, guaranteeing that the worst-case prediction 
error as measured by the scoring rule is minimized (compared to the average prediction
error). When, for example, evaluating ensemble predictions based on Brier's quadratic score 
\citep[see Appendix~\ref{app:scoringrules} for the definition]{brier1950verification},
aggregating predictions with the arithmetic mean is minimax optimal, corresponding to $g$ 
being the identity. For the negative log-likelihood, a geometric mean aggregation is 
minimax optimal, corresponding to $g$ being the natural logarithm.

\section{Background}\label{sec:bg}

We briefly recap classical ensembles and proper scoring rules. Then, we give a 
short overview of deep conditional transformation models as the backbone of our 
proposed transformation ensembles.

\subsection{Ensembles}
\label{subsec:ensembles}

Ensembles aggregate predictions from multiple models. Here, we focus on 
probabilistic predictions, such as conditional CDFs $\pmem_1, \dots, \pmem_\Be$ 
of $\Be$ models. The most commonly used ensemble methods are the classical linear 
and log-linear ensembles, which we discuss in the following.
\begin{definition}[Classical linear ensemble] \label{def:classical}
Let $\pmem_1, \dots, \pmem_\Be$ be $\Be$ CDFs and $w_1, \dots, w_\Be$
be non-negative weights summing to one.
The classical linear ensemble is defined as the point-wise weighted
average of the $\Be$ ensemble members $\pens^c = \sum_\be w_\be \pmemb$. When
using equal weights, the linear ensemble is equivalent to taking the arithmetic
mean, $\Be^{-1}\sum_\be \pmemb$.
\end{definition}
The ensemble distribution $\pens^c$ can be viewed as a mixture distribution with
weights $w_\be$. Note that for linear ensembles it does not matter whether one 
ensembles on the scale of the CDF or the probability density function (PDF), because
$\dens^c(y) = \sum_\be w_\be \dmemb(y) = \frac{\dd}{\dd y} \pens^c(y)$ for
continuous CDFs.

In this article we formulate all ensembles on the scale of the CDF, if it is 
well-defined (\ie for random variables with an at least ordered sample space).
Note, however, that in \eg multi-class classification, it is common to ensemble 
on the density scale. Here, the predicted probabilities for each class 
$p(k) := \Prob(\rY = \ry_k)$ are aggregated via
$\bar p_\Be(k) = \sum_\be w_\be p_\be(k), \; k = 1, \dots, K$. 
Performing linear ensembling on the density scale with deep neural networks is
called deep ensembling (Section~\ref{subsec:related}).

For convex loss functions, the performance of the classical linear ensemble
is always at least as good as the average performance of its members 
(Prop.~\ref{prop:classical}).
\begin{proposition}[\eg \citet{abe2022deep}] \label{prop:classical}
Let $\pmem_1, \dots, \pmem_\Be$ be CDFs with $w_1, \dots, w_\Be$ 
non-negative weights summing to one. Let $L: \calP \to \RR$ be a convex loss function. 
Then, $L\left(\sum_\be w_\be \pmemb \right) \leq \sum_\be w_\be 
L\left(\pmemb\right)$.
\end{proposition}
The claim follows directly from Jensen's inequality by convexity of $L$. 
In particular, this holds for the negative log-likelihood (NLL) and the
ranked probability score (RPS) as loss functions (for definitions, see
Section~\ref{subsec:scoring}).

The arithmetic mean is not the only way to aggregate forecasts. The 
geometric mean (arithmetic mean on the log-scale) is used for log-linear 
ensembles, as defined in the following.
\begin{definition}[Classical log-linear ensemble] \label{def:loglin}
The log-linear ensemble is defined as the point-wise geometric mean of the 
$\Be$ ensemble members $\pens^l = \exp\left(\sum_\be w_\be \log\pmemb\right)$.
\end{definition}
The log-linear ensemble, as defined here, is a special case of quasi-arithmetic
pooling with $g = \log$ on the scale of the CDF, see eq.~\eqref{eq:qa}. 
As mentioned above, commonly (\ie in classification problems with nominal outcomes) 
densities are aggregated in log-linear ensembles which requires scaling by a constant 
such that the ensemble density integrates to one. Regardless of whether log PDFs or
log CDFs are pooled, the ensemble will still score better in terms of negative
log-likelihood than its members do on average (see Appendix~\ref{app:theory}).

\subsection{Scoring rules}
\label{subsec:scoring}

Scoring rules are metrics designed to evaluate probabilistic predictions 
\citep{gneiting2007strictly}. An ideal score judges predictions from a model 
by how faithful they are to the data-generating distribution and thus penalize
overconfident as well as too uncertain predictions. Scoring rules can be 
categorized as to whether they adhere to this ideal (such a score is called proper,
Def.~\ref{def:proper}). First, we define scoring rules.
\begin{definition}[Scoring rule, \citet{gneiting2007strictly}] 
\label{def:scoringrule}
Let $P \in \calP$ be a probability distribution corresponding to a random 
variable with sample space $\calY$. Then, a scoring rule is an extended 
real-valued function $s: \calP \times \calY \to \RR \cup \{-\infty, \infty\}$,
and $s(P, y)$ is the prediction error when predicting $P$ and observing $y$.
\end{definition}
In this article, we take scoring rules to be negatively oriented (\ie smaller 
is better), which is the more natural choice in deep learning because model
training is usually phrased as a minimization problem.
We now define proper scoring rules.
\begin{definition}[Proper scoring rule, \citet{gneiting2007strictly}] 
\label{def:proper}
Let $P, Q \in \calP$ be two probability distributions and $\rY \sim P$.
The scoring rule $s$ is proper if $\Ex_P[s(P, \rY)] \leq \Ex_P[s(Q, \rY)]$,
and strictly proper if equality holds iff $P = Q$. 
\end{definition}
\begin{example}[Brier score, \citet{brier1950verification}]
When predicting a binary outcome $\rY \sim \operatorname{Bern}(\pi)$,
we use the Brier score, $\BS(p, \ry) = (y - p)^2$, to score a prediction $p$.
The expected Brier score is minimized when predicting $p = \pi$, hence the 
Brier score is proper.
\end{example}
Other scoring rules used in this paper are the log-score \citep{good1952rational}
and the ranked probability score \citep[RPS,][see appendix \ref{app:scoringrules}
for more detail]{epstein1969scoring}. The log-score is strictly proper and local 
and defined as
\begin{align}
    \NLL(f, \ry) = - \log f(\ry),
\end{align}
where $f$ is the predicted density of $\rY$. In case an observation is censored 
(\ie set-valued), the NLL is given by 
$\NLL(f, (\ubar\ry, \bar\ry]) = - \log\left(F(\bar\ry) - F(\ubar\ry)\right)$, where $F$
denotes the CDF corresponding to $f$, provided it exists. The log-score is equivalent to 
the negative log-likelihood \footnote{In fact, for continuous random variables, the 
log-score is the only smooth, proper and local score \citep{brocker2007scoring}.}.
Locality refers to $s$ solely being evaluated at the observed $y$ instead of the whole
sample space. The RPS is a global (\ie non-local) proper scoring rule for ordinal outcomes 
$\rY \in \{\ry_1 < \dots < \ry_K\}$ and defined as
\begin{align} \label{eq:rps}
   \RPS(F, \ry) = \frac{1}{K - 1} \sum_{k=1}^{K-1} \left(F(\ry_k) - 
       \mathds{1}(\ry \leq y_k) \right)^2,
\end{align}
where $F$ is the predicted CDF for $\rY$. The RPS can be viewed as a sum of 
cumulative Brier scores for each category, highlighting that it is a global score.

When training a model on data $\{(\ry_i, \rx_i)\}_{i=1}^n$ and the goal is 
probabilistic prediction, it is advantageous to minimize an empirical proper score, 
$\bar s_n = \frac{1}{n} \sum_{i = 1}^n s(P_i, y_i)$, 
as, for instance, in maximum likelihood \citep{fisher1922mathematical} or continuous RPS 
learning \citep{gneiting2005calibrated,berrisch2021crpslearning}.

\subsection{Deep conditional transformation models}
\label{subsec:ontram}

Although our proposed ensembling scheme (see Def.~\ref{def:trafoensemble}) is applicable
to any probabilistic deep neural network, it is most beneficial when all members are deep 
conditional transformation models. Therefore, we briefly recap transformation 
models \citep{hothorn2014conditional} and their extensions involving deep neural networks.
Transformation models generalize and extend several well-known models, such as the normal 
linear regression model, cumulative ordinal regression, and the Cox proportional hazards 
model \citep{hothorn2018most}. In their general form, transformation models are 
distribution-free \citep{bell1964characterization}. That is, no predefined conditional 
outcome distribution has to be chosen.

Transformation models construct the conditional CDF of an outcome $\rY$ 
given data $\calD$ as follows. The outcome variable $\rY$ is passed through a
transformation function $\h$, which is estimated from the data, such that the 
transformed outcome variable $\h(\rY \given \calD)$ follows a parameter-free 
target CDF $\pZ$ (see Fig.~\ref{fig:tram}), 
\begin{align} \label{eq:tram}
    \pmem(\ry \given \calD) = \pZ(\h(\ry \given \calD)).
\end{align}
\begin{figure}[!ht]
    \centering
    \resizebox{\textwidth}{!}{
    \begin{tikzpicture}
    \node (cont) {\includegraphics[width=0.49\textwidth]{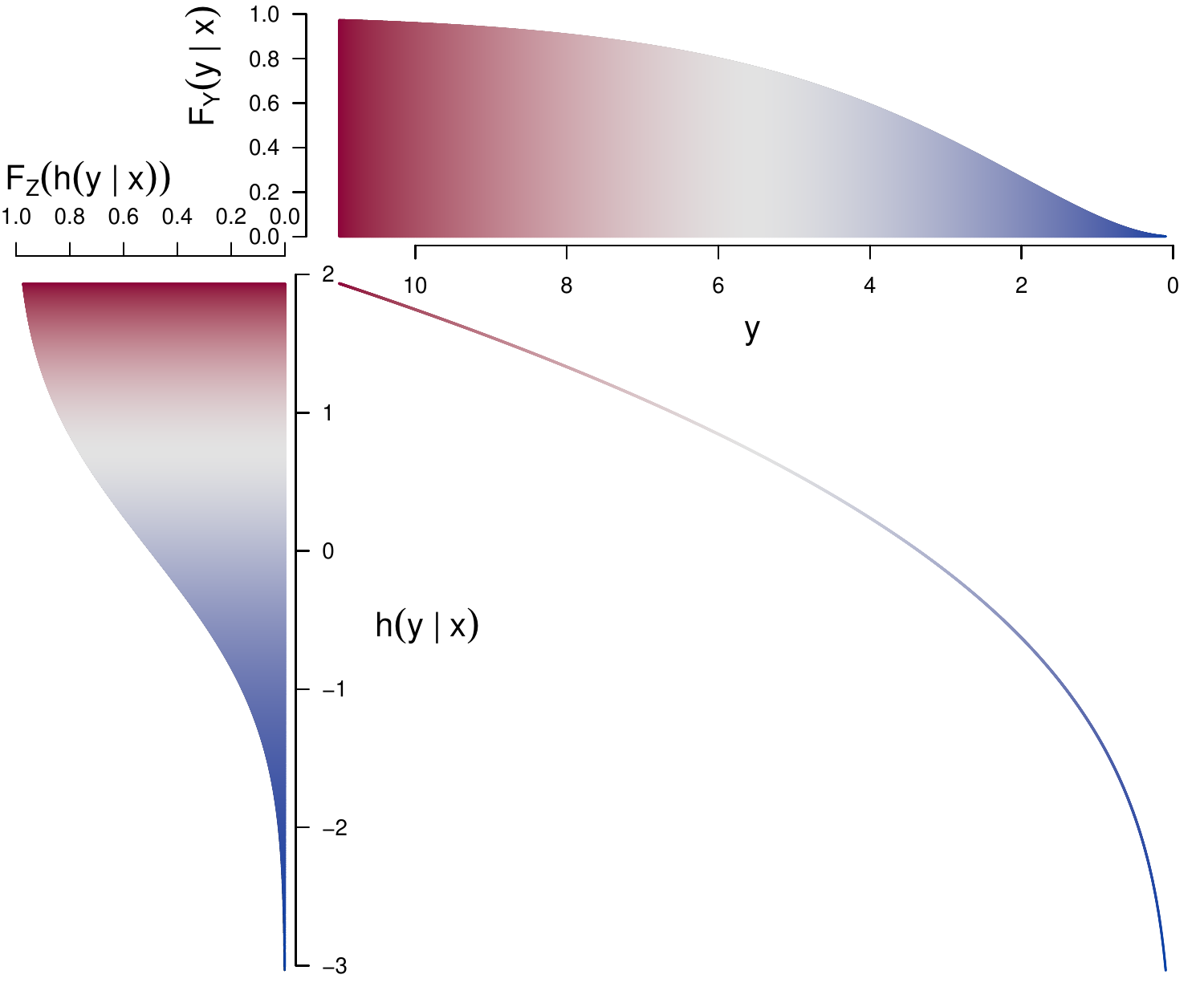}};
    \node[above=0cm of cont.north west] (a) {\textsf{A}};
    \node[right=0cm of cont] (ord) {\includegraphics[width=0.49\textwidth]{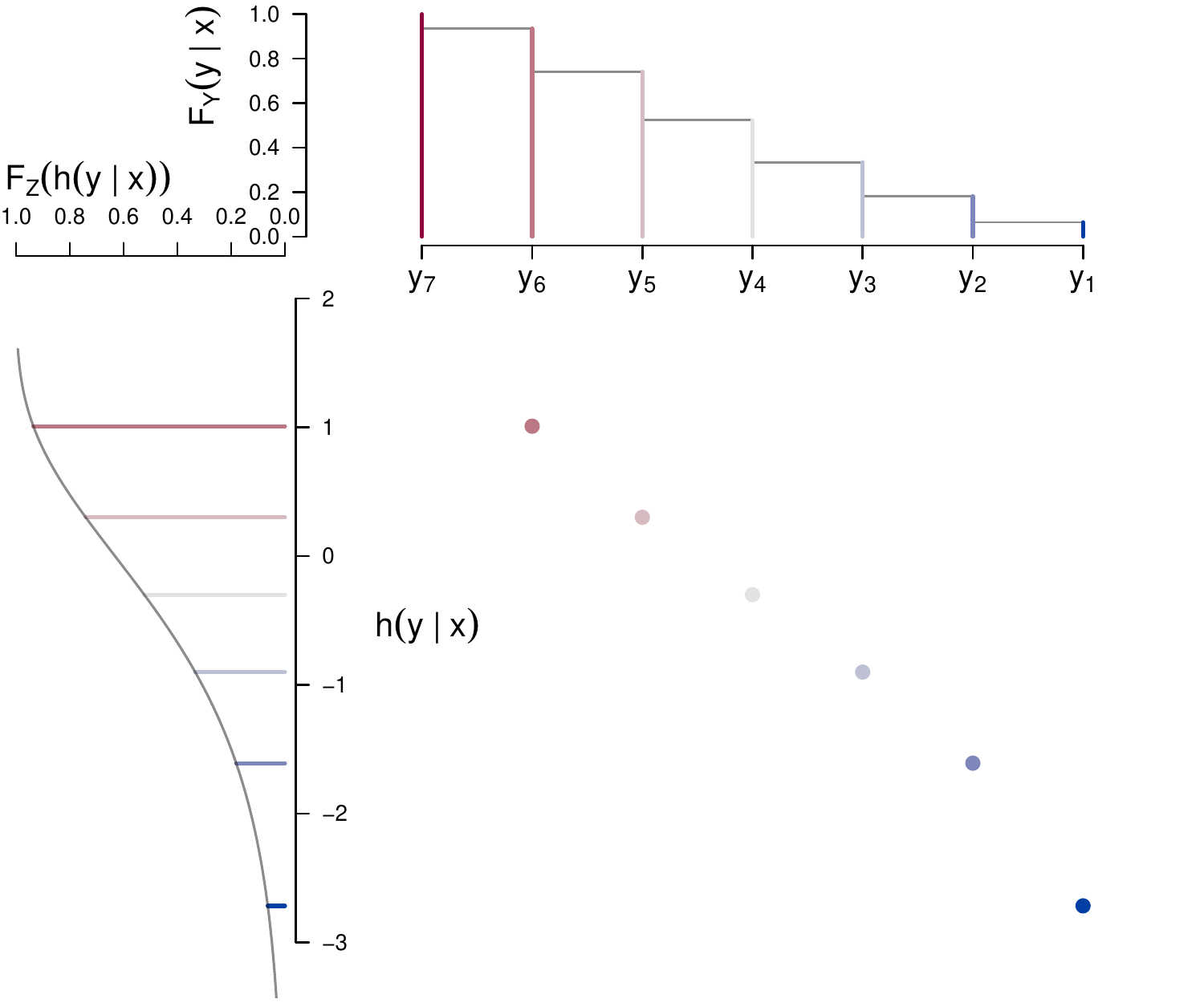}};
    \node[above=0cm of ord.north west] (b) {\textsf{B}}; 
    \end{tikzpicture}}
    \caption{Transformation models for continuous (\textsf{A}) and ordinal 
    (\textsf{B}) outcomes are constructed such that
    $\Prob(\rY \leq \ry \given \rx) =
    \Prob(\h(\rY \given \rx) \leq \h(\ry \given \rx))$ holds.
    The monotone increasing continuous or discrete transformation function 
    (lower right panel in \textsf{A} and \textsf{B}) is 
    estimated from the data by minimizing an empirical proper score.}
    \label{fig:tram}
\end{figure}
In transformation models, $\pZ$ is required to be a CDF with log-concave density and 
the transformation function $\h$ needs to be a monotone non-decreasing function in $\ry$.
Common choices for $\pZ$ are the standard normal, standard logistic, and standard minimum 
extreme value distribution. Transformation models are closely related to normalizing flows 
in deep learning \citep{papamakarios2021normalizing}.

\citet{sick2020deep} demonstrate how transformation models for continuous outcomes can 
be set up with neural networks to handle tabular and/or image data, however the authors 
focused solely on prediction power and not on interpretability. \citet{baumann2020deep}
and  Kook \& Herzog et al. (\citeyear{kook2020ordinal}) explore intrinsic interpretability
of deep transformation models. In the latter, the authors propose models for semi-structured
input data and ordinal outcomes ({\sc ontram}s).
In {\sc ontram}s the transformation function is parameterized via a neural network and 
the model parameters are optimized jointly by minimizing the NLL.
This way of combining traditional statistical methods with deep learning allows 
the construction of interpretable, yet powerful prediction models. For instance, having 
access to tabular data $\rx$ and image data $\B$, we can formulate models for the 
conditional distribution of $\rY$, such as
\begin{align*}
    \pY(\ry \given \rx, \B) &= \pZ\left(
    \hY(\ry) - \rx^\top\shiftparm - \eta(\B) \right), \quad \text{or} \\ 
    \pY(\ry \given \rx, \B) &= \pZ\left(
    \hY(\ry \given \B) - \rx^\top\shiftparm\right).
\end{align*}
The first model assumes a simple linear shift ($\rx^\top\shiftparm$) for the tabular 
data and a complex shift ($\eta(\B)$) for 
the image data, while the second model allows full flexibility in the image data, but 
still assumes linear shift effects for the tabular data. When using 
$\pZ(\rz) = \expit(\rz) = (1 + \exp(-\rz))^{-1}$, 
the shift terms can be interpreted as log odds-ratios. Then the first model assumes 
proportional odds for both $\rx$ and $\B$. The second model lifts this restriction 
on the image component. The components of the transformation function are
controlled by (deep) neural networks, \ie a convolutional neural network for $\B$ and a
single-layer NN for $\rx$ (for details see Kook \& Herzog et al. \citeyear{kook2020ordinal}).

\section{Transformation ensembles} \label{sec:trafoensemble}

Transformation models can be used as flexible, yet interpretable prediction models
with semi-structured data, making them an attractive choice of model class.
To further improve their prediction power, one could use classical deep ensembling.
However, when pooling transformation models linearly, the ensemble looses the structural
assumptions of its individual members (\eg proportional odds), and thus intrinsic 
interpretability is lost in general (see Fig.~\ref{fig:spaces} \textsf{A}). For instance, 
the average of two Gaussian densities is generally not Gaussian anymore and neither 
unimodal nor necessarily symmetric. 
To improve upon the black-box character of classical ensembles, we propose 
\emph{transformation ensembles}. Transformation ensembles are specifically tailored
towards transformation models for which predictions are aggregated on the scale of 
the transformation function $\h$. That is, the transformation ensemble with members 
$\pmemb = \pZ\circ\hb$ is given by $\pens^t = \pZ \left(\sum_\be w_\be \hb\right)$.
\begin{definition}[Transformation ensemble] \label{def:trafoensemble}
Let $\pmem_1, \dots, \pmem_\Be$ be CDFs and $w_1, \dots, w_\Be$ be non-negative weights
summing to one. Let $\pZ : \RR \to [0, 1]$ be a continuous CDF with quantile function 
$\pZ^{-1}$ and log-concave density. The transformation ensemble is defined as
\begin{align} \label{eq:trafoquap}
  \pens^t = \pZ \left(\sum_\be w_\be \pZ^{-1} \circ \pmemb\right).
\end{align}
Further, if each member is a transformation model, \ie $\pmemb = \pZ \circ \hb$,
the transformation ensemble simplifies to $\pens^t = \pZ \left(\sum_\be w_\be \hb\right)$.
For continuous outcomes, the transformation ensemble density is given by the transformation
model density evaluated at the average transformation function, \ie $\dZ(\hens(\ry))\hens'(\ry)$,
where $\hens = \sum_\be w_\be \hb$ and $\hens'(\ry) = \frac{\dd}{\dd\ry} \hens(\ry)$.
\end{definition}
Note that transformation ensembles require only the existence of the CDF of each 
ensemble member. These CDFs may stem from any deep neural network, which need not 
necessarily be a transformation model. A transformation ensemble can still be
constructed after picking any continuous CDF $\pZ$.

In the following, we show that transformation ensembles remain as intrinsically 
interpretable as their members, while still producing provably better than average 
predictions, in case their members are transformation models.
\begin{proposition} \label{prop:closure}
The family of transformation models, 
$\calF = \{\pY(\ry \given \calD) = \pZ(\h(\ry\given\calD)) : \h \mbox{ monotone increasing in } \ry\}$, 
is closed under transformation ensembling. That is, if $\pmemb \in \calF$ for all $\be$,
then also $\pens^t \in \calF$. 
\end{proposition}
This result follows directly from Def.~\ref{def:trafoensemble} and is illustrated 
in Fig.~\ref{fig:spaces}. Albeit straightforward, Prop.~\ref{prop:closure} has important
consequences for the interpretability of transformation ensembles. The result is especially
interesting for more special cases, such as linear and semi-structured deep transformation
models, as illustrated in the following example.
\begin{example}[Semi-structured deep transformation models]
The model $\pmemb = \pZ \circ \hb$ with $\pZ(\rz) = 1 - \exp(-\exp(\rz))$ being 
the standard minimum extreme value CDF and
$\hb = \basisy(\ry)^\top\parm_\be - \rx^\top\shiftparm_\be - \eta_\be(\B)$,
assumes proportional hazards for both structured ($\rx$) and unstructured
data ($\B$). The resulting transformation ensemble retains the proportional
hazard assumption, because 
$\hens(\ry \given \rx, \B) = \basisy(\ry)^\top\ensparm - \rx^\top\ensshiftparm 
- \enscomplex(\B)$ averages the predicted log cumulative-hazards. In contrast,
random survival forests \citep{ishwaran2008random} aggregate cumulative hazards 
without focusing on interpretability.
Note that the outputs of the convolutional neural networks, $\eta_\be(\B)$, are 
averaged, not their weights. The classical linear ensemble would not preserve 
the model structure, because $\sum_\be w_\be \pZ(\hb(\ry \given \rx, \B)) \neq 
\pZ\left(\sum_\be w_\be \hb(\ry \given \rx, \B)\right)$,
in general.
\end{example}
%
\begin{figure}[!ht]
    \centering
    \resizebox{0.49\textwidth}{!}{%
    \Large
    \begin{tikzpicture}
        \draw node[anchor=west] at (-2.5, 4) {\textsf{A} {\normalsize\hspace{0.5cm} 
        Classical ensembles}};
        \draw[smooth cycle, tension=0.4, fill=white, pattern color=purple!30, 
        pattern=north east lines, opacity=0.9] plot coordinates{(2,3) 
        (-1.5,0) (3,-4) (7,1)} node at (5,2.8) {$\calG$};
        \draw[smooth cycle, tension=0.4, fill=white, pattern color=brown!40, 
        pattern=north west lines, opacity=0.9] plot coordinates{(2,2) 
        (-0.5,0) (3,-2) (5,1)} node at (3,2.3) {$\calF$};
        \draw node at (1, 1) {$\boldsymbol\cdot \pmem_1$};
        \draw node at (2.5, 1.5) {$\boldsymbol\cdot \pmem_2$};
        \draw node at (2.0, 0.5) {$\boldsymbol\cdot \pmem_3$};
        \draw[color=red] node at (1.5, 2.3) {$\boldsymbol\cdot \pens^c$};
    \end{tikzpicture}}
    \resizebox{0.49\textwidth}{!}{%
    \Large
    \begin{tikzpicture}
        \draw node[anchor=west] at (-2.5, 4) {\textsf{B} {\normalsize\hspace{0.5cm}
        Transformation ensembles}};
        \draw[smooth cycle, tension=0.4, fill=white, pattern color=purple!30, 
        pattern=north east lines, opacity=0.9] plot coordinates{(2,3) 
        (-1.5,0) (3,-4) (7,1)} node at (5,2.8) {$\calG$};
        \draw[smooth cycle, tension=0.4, fill=white, pattern color=brown!40, 
        pattern=north west lines, opacity=0.9] plot coordinates{(2,2) 
        (-0.5,0) (3,-2) (5,1)} node at (3,2.3) {$\calF$};
        \draw node at (1, 1) {$\boldsymbol\cdot \pmem_1$};
        \draw node at (2.5, 0.5) {$\boldsymbol\cdot \pmem_2$};
        \draw node at (1.5, 0.1) {$\boldsymbol\cdot \pmem_3$};
        \draw[color=red] node at (1.9, 1.0) {$\boldsymbol\cdot \pens^t$};
    \end{tikzpicture}}
    \caption{Comparing classical and transformation ensembles. \textsf{A}: In
    classical ensembles, structural assumptions on the individual $\pmemb \in \calF$
    do not necessarily carry over to the ensemble $\pens^c \in \calG \supset \calF$.
    Note, that there exist special cases in which the classical ensemble may remain
    in $\calF$, for instance if all ensemble members are equal, 
    $\pmemb = \pmem,\;\forall\be$.
    \textsf{B}: The transformation ensemble, on the other hand, ensures that
    $\forall \be \; \pmemb \in \calF \implies \pens^t \in \calF$.}
    \label{fig:spaces}
\end{figure}
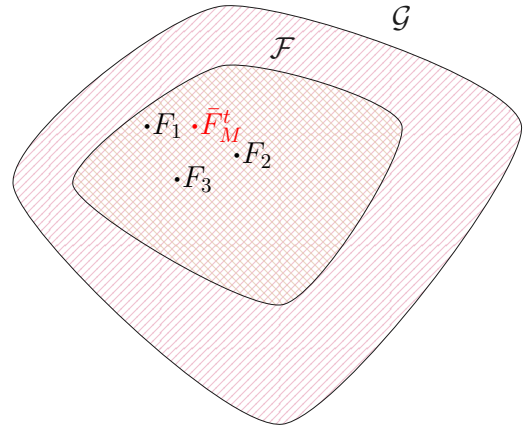

In Prop.~\ref{prop:interpretable} below, we show that the NLL of the transformation 
ensemble is uniformly better than the average NLL of its members. This is analogous 
to the improved prediction performance of classical linear and log-linear ensembles 
w.r.t.~the NLL. A proof is given in Appendix~\ref{app:theory}.  
\begin{proposition} \label{prop:interpretable}
Let $\pZ \circ \h_1, \dots, \pZ \circ \h_\Be$ be transformation model CDFs and
$w_1, \dots, w_\Be$ be non-negative weights summing to one. Let $\NLL: \calP \to \RR$
denote the negative log-likelihood induced by the transformation model $\pZ \circ \h$.
Then, $\NLL\left(\pZ(\sum_\be w_\be \hb)\right) \leq 
\sum w_\be \NLL\left(\pZ\circ\hb \right)$.
\end{proposition}
\subsection{Comparison between transformation and classical ensembles}
\label{subsec:comparison}

We juxtapose classical (linear and log-linear) and transformation ensembles in 
terms of their key features, interpretability, prediction performance and 
uncertainty quantification. In addition, we present minimax optimality results 
for a subset of transformation ensembles.

\paragraph{Interpretability and prediction performance}
Fig.~\ref{fig:ens} shows the resulting linear ensemble, log-linear ensemble and
transformation ensemble of five logistic densities. All three ensembles achieve 
uniformly better negative log-likelihoods (Fig.~\ref{fig:ens} \textsf{B}) than 
the ensemble members do on average (solid black line). The linear ensemble 
has thicker tails and thus is more uncertain than the log-linear and transformation 
ensembles (Fig.~\ref{fig:ens} \textsf{A}). In addition, the linear ensemble does no
longer resemble a logistic distribution, while the transformation ensemble preserves 
the distribution family. The transformation ensemble is similar to the log-linear 
ensemble in this example and both achieve a better NLL for $\ry$ close to zero.
\begin{figure}[!ht]
    \centering
    \includegraphics[width=0.97\textwidth]{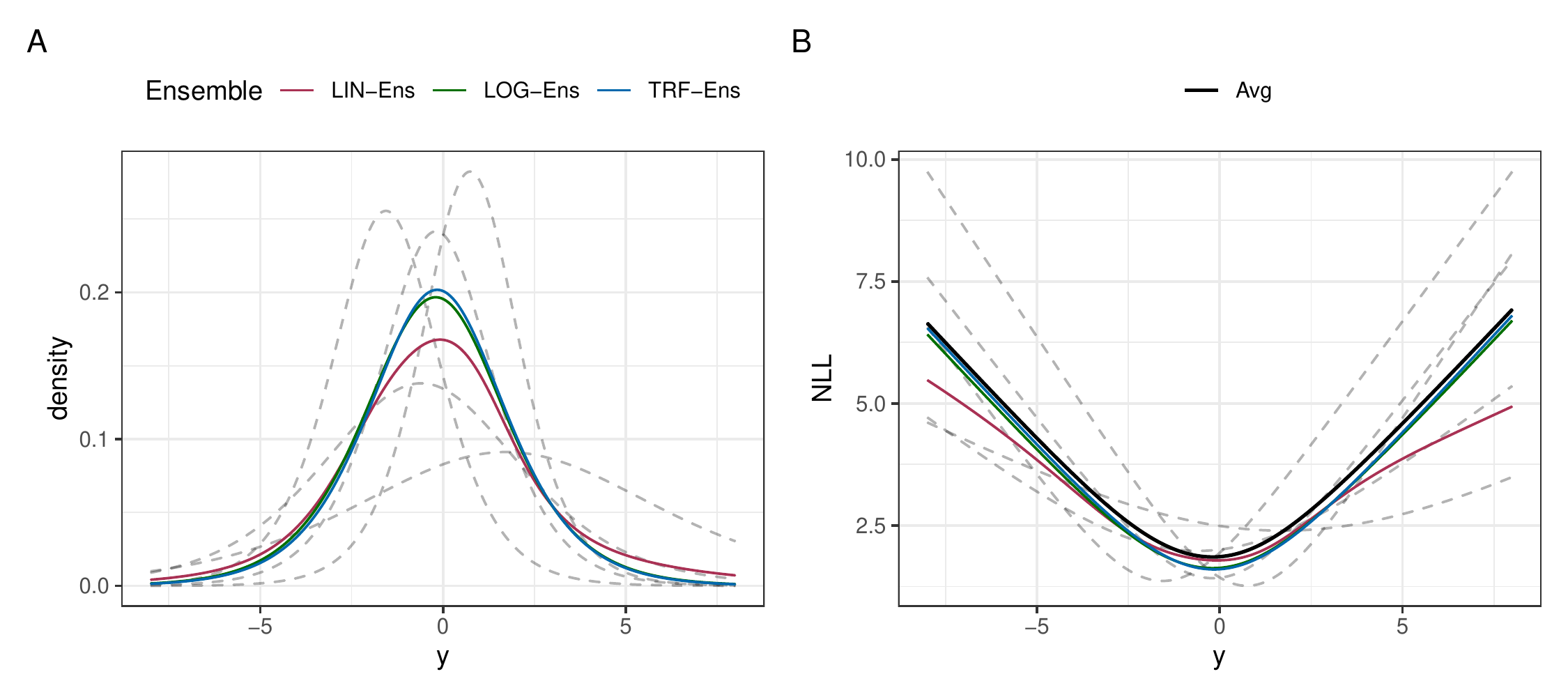}
    \caption{Illustration of linear (LIN-Ens), log-linear (LOG-Ens), and 
    transformation (TRF-Ens) ensemble densities (\textsf{A}) and negative 
    log-densities (\textsf{B}). Five logistic densities (dashed lines) are 
    aggregated into an ensemble using the approaches outlined above.}
    \label{fig:ens}
\end{figure}

An important question is whether increasing flexibility of the model class in 
classical ensembling is necessary for the commonly observed benefits in 
prediction performance. If the increased model complexity was necessary, 
transformation ensembles would not be able to compete with the classical ensembles,
because they remain in the same model class as their members. 
In Section~\ref{sec:results} we present empirical evidence that the increased
complexity is not necessary. All types of ensembles show considerable
improvement over the average performance of their members, but all types of
ensemble perform roughly on par.

Even when using classical ensembles with transformation models as members, we can 
use the transformation ensemble approach to judge whether the ensemble deviates 
from the structural assumptions of its members. To do so, we transform the $\Be$ 
predicted CDFs to the scale of the empirical transformation function, 
$\{\hb(\ry_i \given \rx_i) = \pZ^{-1}\circ \pmemb(\ry_i \given \rx_i)\}_{i=1}^{n}$,
and plot those against the likewise transformed ensemble predictions
$\{\pZ^{-1}\circ\pens(\ry_i \given \rx_i)\}_{i=1}^n$, where $\pens(\ry_i \given \rx_i)$
is for example the classical linear ensemble (see Fig.~\ref{fig:trafos} \textsf{A}).
Observations close to the diagonal indicate that the ensemble does not deviate
far from the model class of its members. In this example, notable deviations 
from the diagonal can be observed for smaller values of $\ry$, indicating the
classical ensemble deviates from the model class of the members, which assumes 
additivity on the probit scale.
\begin{figure}[!ht]
    \centering
    \includegraphics[width=0.97\textwidth]{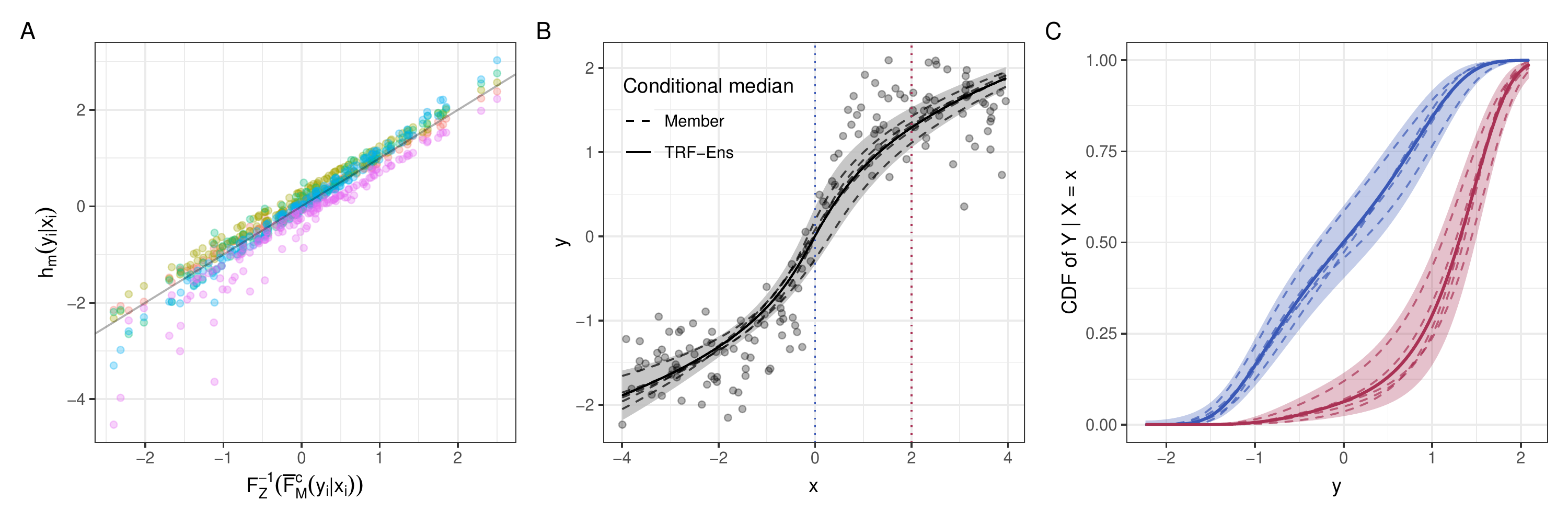}
    \caption{Transformation ensemble approach for checking model structure 
    and quantifying uncertainty. \textsf{A}: Visual comparison of how 
    far an ensemble deviates from the model class of its members on the
    scale of $\pZ^{-1}$, in this case the probit scale. \textsf{B}: 
    Epistemic uncertainty of the transformation ensemble (TRF-Ens) in 
    predicting the conditional median of $\rY \given x$.
    \textsf{C}: Epistemic uncertainty in the conditional distribution of 
    $\rY \given X$ is displayed for $X=0$ (blue) and $X=2$ (red). Individual 
    members are shown as gray dashed
    lines, and the transformation ensemble CDF with point-wise epistemic uncertainty 
    as solid black lines. The data to produce this figure were simulated from the 
    model $Y = 3 \expit(3x) + \epsilon$, where $\epsilon \sim \ND(0, 0.6^2)$.}
    \label{fig:trafos}
\end{figure}

\paragraph{Uncertainty quantification}
We now pose the related question of whether increasing model complexity via classical 
ensembling is necessary for the frequently noted improvement in epistemic 
uncertainty quantification in the classical ensemble. Transformation ensembles yield 
epistemic uncertainty estimates on the scale of the transformation function.
Epistemic uncertainty can be displayed, for instance, when predicting the conditional
$\alpha$-quantile, $\bar q_\Be(\rx;\alpha) := \hens^{-1}(\pZ^{-1}(\alpha) \given \rx)$.
In Fig.~\ref{fig:trafos} \textsf{B}, we show the conditional median, \ie $\alpha = 0.5$,
as a point estimate along with the point-wise standard deviation
$\bar q_\Be(\rx; 0.5) \pm 2 \operatorname{sd}(q_\be(\rx; 0.5))$
as a shaded area in the scatter plot of $\ry$ versus $\rx$.
For the prediction at a single $\rx$, the epistemic uncertainty can also be transformed 
back to the CDF scale (Fig.~\ref{fig:trafos} \textsf{C}). Here, the epistemic uncertainty 
is depicted as $\pZ(\hens \pm 2\operatorname{sd}(\hb))$, where $\operatorname{sd}(\hb)$ 
denotes the point-wise standard deviation of the predicted transformation functions $\hb$.
Besides uncertainty in function space (transformation function, CDF, conditional median),
transformation ensembles also yield epistemic uncertainty for the parameter space of additive
linear predictors. For instance, if the model includes $\rx^\top\shiftparm$, one can 
compute the epistemic uncertainty for each $\beta_p$, $p = 1, \dots, P$ using 
asymptotic or bootstrap confidence intervals.

A major advantage of ensembling distributional regression models which are fitted 
using a loss derived from a proper score is that the members already quantify 
aleatoric uncertainty. Calibration of single models assesses how well the model 
quantifies aleatoric uncertainty. A model is (strongly) calibrated if its predictions 
match the population frequency of events, \ie if
$\Prob(\rY \leq \ry \given \hat\pmem) = \hat\pmem(\ry)$. Calibration of ensembles
assess overall uncertainty, \ie both types of uncertainty, aleatoric and epistemic, 
simultaneously.

\paragraph{Minimax optimality}
Commonly practitioners aim for the best performing model. However, in many applications
optimal worst-case prediction errors are sought 
\citep[\eg in robust statistics,][]{huber2004robust}. 
For instance, in forecasting extreme weather events practitioners may be willing to
sacrifice average prediction performance to mitigate worst-case prediction errors
\citep{gumbel1958statistics}.
There is a strong similarity between transformation ensembles in eq.~\eqref{eq:trafoquap} 
and quasi-arithmetic pooling in eq.~\eqref{eq:qa}, for which minimax properties were 
shown for nominal outcomes \citep{neyman2021proper}. While quasi-arithmetic pooling is
defined for densities of nominal outcomes, transformation ensembles act on the CDF of 
outcomes with an ordered sample space. However, for the special case of binary outcomes
both aggregation methods coincide. Hence, for binary outcomes transformation ensembles 
are guaranteed to minimize worst-case prediction error in terms of NLL. The result 
follows from Theorem~4.1 in \citet{neyman2021proper}.
\begin{corollary}[Minimax optimality of transformation ensembles] \label{thm:minimax}
Let $p_1, \dots, p_\Be$ be predicted probabilities for success in a binary outcome, 
\ie $p_\be = \Prob_\be(\rY = 1 \given \calD)$ and $w_1, \dots, w_\Be$ be non-negative
weights summing to one. Then $\bar p_\Be^t = \expit(\sum_\be w_\be\logit(p_\be))$ 
minimizes
\begin{align*}
    \max_\ry \NLL(p, \ry) - \sum_{\be = 1}^\Be w_\be \NLL(p_\be, \ry).
\end{align*}
\end{corollary}
In words, transformation ensembles for binary outcomes with $\pZ = \expit$ (standard logistic)
are minimax optimal w.r.t. the NLL. Note that the result is independent of the type of ensemble 
members, \ie for minimax optimality to hold, the members do not need to be 
transformation models. The proof of Corollary~\ref{thm:minimax} is given in 
Appendix~\ref{app:theory}. There, we also prove minimax optimality of linear pooling in terms 
of RPS and Brier score.

\section{Experimental setup} \label{sec:setup}

We evaluate transformation ensembles in terms of prediction performance and
calibration on several publicly available, semi-structured data sets with nominal,
binary, and ordinal outcomes. We compare transformation ensembles, that preserve 
structure and interpretability of its members, to state-of-the-art ensembling methods 
(linear, log-linear ensembles, see Section~\ref{subsec:ensembles}) where the structure 
of the members is not preserved. In the following, we describe the different data sets 
and models used in our experiments.

\subsection{Data sets}
\label{subsec:datasets}

\paragraph{Melanoma} 
The publicly available melanoma data set \citep{dataset:melanoma} contains 
skin lesion color images of dimension $128 \times 128 \times 3$ along with 
age information of 33'058 patients. The response is binary and highly 
imbalanced (98.23\% of all skin lesions are benign and 1.77\% malignant).

\paragraph{UTKFace} 
The publicly available UTKFace data set contains facial images from people of 
various age groups. Here, age is treated as an ordinal outcome using seven categories,
0--3 ($n = 1'894$), 4--12 ($n = 1'519$), 13--19 ($n = 1'180$), 20--30 ($n = 8'068$),
31--45 ($n = 5'433$), 46--61 ($n = 3'216$) and $>61$ ($n = 2'395$) \citep{das2018}.
In addition, the data set contains sex (female, male) as a feature. To compare against 
the results reported in \citet{kook2020ordinal}, we use the same cropped version of the 
images reported therein. We simulate 10 additional covariates using the same simulation
scheme as in \citet{kook2020ordinal} with effect sizes $\pm \log 1.2$ for $X_{\{2, 3\}}$,
$\pm \log 1.5$ for $X_{\{5, 6\}}$, and $0$ for the remaining covariates.

\paragraph{MNIST} 
The MNIST data set is a publicly available data set \citep{lecun2010mnist} 
containing 60'000 gray-scaled training images of handwritten digits ranging 
from 0 to 9 with dimension $28 \times 28$. In order to be able to define a 
CDF for the outcome, we arbitrarily fix the order to $0, 1, \dots, 9$. We
choose to present results for the MNIST data set, to illustrate that the
transformation ensemble approach is applicable to unordered outcomes too.
All results for the MNIST data set are presented in Appendix~\ref{app:figures}.

\subsection{Models}
\label{subsec:models}

For all data sets, we train several (ordinal) neural network transformation 
models (see Section~\ref{subsec:ontram}) of varying intrinsic 
interpretability and flexibility. The data sets feature nominal, binary, and ordinal 
outcomes, all of which can be handled by \textsc{ontram}s.
\begin{table}[!ht]
    \centering
    \caption{Overview of data sets, outcomes, and models fitted thereon. 
    MNIST features a nominal outcome, for which we fit the most flexible 
    \cib{} model. The melanoma data set comes with a binary outcome for which 
    we can fit a semi-structured model (\ciblsx), two models using only one 
    of the two modalities, and the unconditional model (\si). For UTKFace, 
    age is split into ordered categories and several models of varying 
    complexity are fitted. All models use $\pZ = \expit$ as the target 
    distribution and hence all components of the transformation function 
    are interpretable on the log-odds scale. CI: Complex intercept. CS: Complex
    shift. SI: Simple intercept. LS: Linear shift. A subscript indicates the
    input data for a model term, \eg \csb{} is a neural network with image input
    modelling a complex shift effect.}
    \label{tab:models}
    \resizebox{0.9\textwidth}{!}{%
    \begin{tabular}{ccccc}
    \toprule
         \bf Data set & \bf Outcome & \bf Type & \bf Model name & 
            \bf Transformation function \\ \midrule
         \multirow{4}{*}{Melanoma} & 
         \multirow{4}{*}{benign/malign} & 
         \multirow{4}{*}{Binary} & 
         \ciblsx & $\eparm(\B) - \rx^\top\shiftparm$ \\
         & & & \cib & $\eparm(\B)$ \\ 
         & & & \silsx & $\eparm - \rx^\top\shiftparm$ \\ 
         & & & \si & $\eparm$ \\ 
         \midrule
         \multirow{6}{*}{UTKFace} & 
         \multirow{6}{*}{age groups} & 
         \multirow{6}{*}{Ordinal} & 
         \ciblsx & $\eparm_k(\B) - \rx^\top\shiftparm$ \\
         & & & \sicsblsx & $\eparm_k - \eta(\B) - \rx^\top\shiftparm$ \\ 
         & & & \cib & $\eparm_k(\B)$ \\ 
         & & & \sicsb & $\eparm_k - \eta(\B)$ \\ 
         & & & \silsx & $\eparm_k - \rx^\top\shiftparm$ \\ 
         & & & \si & $\eparm_k$ \\                   
         \midrule 
         \multirow{1}{*}{MNIST} & 
         \multirow{1}{*}{Digits 0--9} &
         \multirow{1}{*}{Nominal} &
         \cib & $\eparm_k(\B)$ \\
         \bottomrule
    \end{tabular}}
\end{table}
The degree of intrinsic interpretability can be controlled by parametrizing the 
transformation function in different ways (see Table~\ref{tab:models}). We follow 
the model nomenclature in \citet{kook2020ordinal}, where transformation functions
consist of an intercept (SI: simple intercept, \ie intercepts do not depend on input data, 
CI: complex intercept, intercepts depend on input data) and potentially shift terms (LS: 
linear shift, additive linear predictor of tabular features, CS: complex shift,
additive but flexible predictor depending on either tabular or image data).
A subscript indicates which part of the model depends on which input modality.
Note that for models with binary outcomes, the \cib{} and \sicsb{} are equivalent,
unless the complex shift term is restricted to have mean zero. Therefore, we only
fit the \cib{} version of these models which are reparametrizations of classical 
image convolutional networks with softmax last-layer activation.

\paragraph{Training}
All models were fitted by minimizing NLL or RPS, both of which are proper scores 
(Section~\ref{subsec:scoring}). In contrast to NLL, RPS is a global score which is
bounded and explicitly takes the natural order of the outcome into account. For binary 
responses, RPS reduces to the Brier score. We apply all ensemble methods to five 
instances of the same model for six random splits of each data set. Training 
procedures and model architectures are described in Appendix~\ref{app:experiments}
in more detail.

\paragraph{Tuning ensemble weights} 
Instead of equal weighting or top-$K$ ensembling, one can tune the ensemble weights
such that the prediction performance w.r.t. a proper score is optimized on a hold-out 
data set. Tuning the composition of an ensemble in such a way is related to stacking
\citep{wolpert1992stacked,breiman1996stacked}.
The hold-out set could be the validation set when splitting the data randomly or
the validation fold during a $k$-fold cross-validation. Concretely, we choose the 
weights by solving
\begin{align}\label{eq:weights}
\begin{split}
    \min_{w \geq 0 } &\sum_{i=1}^n s(\pensi^w, y_i), \\
    \mbox{s.t. } &\sum_m w_m = 1,
\end{split}
\end{align}
for some proper score $s$ and predictions from the weighted ensemble $\pensi^w$. 

\paragraph{Evaluation}
For all data sets, we evaluate the prediction performance via proper scores (NLL, RPS or BS, 
Section~\ref{sec:bg}) and discriminatory performance (accuracy, AUC and Cohen's quadratic 
weighted kappa). We investigate uncertainty quantification of individual models and ensembles 
using calibration plots. In addition, we report calibration in the large and calibration 
slope in Appendix~\ref{app:figures}.
For nominal and ordinal outcomes, these calibration metrics are computed per class. 
See Appendix~\ref{app:experiments} for the exact definitions of the above metrics.

\section{Results and discussion} \label{sec:results}

We apply linear, log-linear, and transformation ensembling to three data sets with 
a nominal, binary, and ordered outcome, respectively (see Section~\ref{subsec:datasets}).
We refer to the three ensemble types as LIN-Ens, LOG-Ens, and TRF-Ens in all figures,
respectively\footnote{All code for reproducing the results is available on GitHub 
\url{https://github.com/LucasKook/interpretable-deep-ensembles}.}.
The focus of the comparison is on the performance 
difference between ensemble types, but also between the different models, which vary 
in their degree of flexibility and interpretability (see Tab.~\ref{tab:models}).
Average ensemble test performance and bootstrap confidence intervals based on six 
random splits of the data are shown for all models. Beside prediction performance,
we discuss interpretability and calibration of the different models and ensembles.

\paragraph{Melanoma}
Four models of different flexibility and interpretability are fitted
(see Table~\ref{tab:models}) with the melanoma data in order to predict the conditional
probabilities of the unbalanced binary outcome benign or malignant lesion given a
person's age and an image of the skin lesion.
\begin{figure}[!ht]
    \centering
    \includegraphics[width=0.90\textwidth]{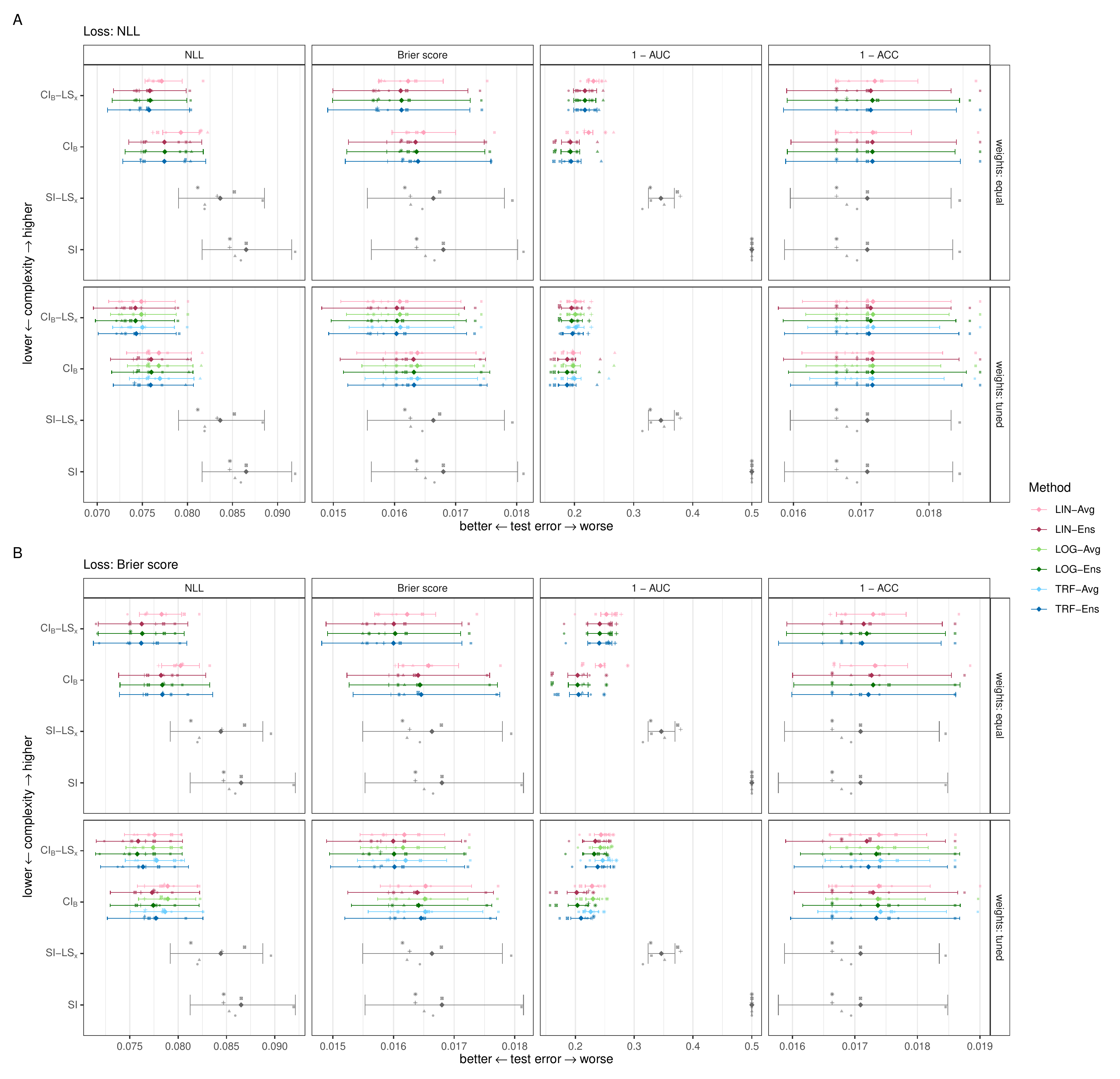}
    \caption{Performance estimates on the melanoma data set.
    The classical linear (LIN-Ens), classical log-linear (LOG-Ens),
    and transformation (TRF-Ens) ensemble test error is shown for 
    negative log-likelihood (NLL), Brier score, discrimination error
    ($1-\operatorname{AUC}$) and classification error 
    ($1-\operatorname{ACC}$). The average ensemble test error and 
    95\% bootstrap confidence intervals are depicted for six random splits
    (indicated by different symbols) of the data. In the upper panels 
    ensemble members are weighted equally and in the lower panels weights are
    tuned to minimize validation loss. In case of equal weights, the average 
    coincides for all ensemble types (LIN-Avg). Models are fitted by minimizing 
    NLL (\textsf{A}) or Brier score (\textsf{B}). Note the different scales for 
    \textsf{A} and \textsf{B}.}
    \label{fig:mela:perf}
\end{figure}

We now see empirically that all ensemble methods (LIN-Ens, LOG-Ens, TRF-Ens)
result in a higher test performance w.r.t. the two proper scores NLL (as shown 
in Prop.~\ref{prop:interpretable}) and RPS compared to their members' average 
(Avg, see Fig.~\ref{fig:mnist:perf}). When additionally tuning the ensemble weights
on the validation set, test prediction and discrimination performance improve further. 
Test performance indicates for all ensemble methods that both input modalities, a
person's age and appearance of the lesion, aid in predicting the risk of a 
malignant \vs a benign skin lesion (see Fig.~\ref{fig:mela:perf}). While the image data
(\cib{}) seems to be most important for prediction, including age (\ciblsx{}) further
improves prediction performance (NLL, Brier score). The unconditional model (SI) achieves
a NLL and Brier score close to that of the model (\silsx{}) based on age alone. However,
\silsx{} lacks in discriminatory ability (AUC), which suggests that the seemingly high 
performance of the unconditional model is due to the highly imbalanced outcome. Performance 
of the individual ensemble members and performance relative to the \silsx{} model are 
shown in Appendix~\ref{app:figures}.

For the best performing \ciblsx{} model all three ensemble methods improve compared 
to the average performance of the individual models and result in similarly low 
test error (NLL, Brier score). 
The positive effect of optimizing the ensemble weights seems to be more
pronounced when there is more variation in the individual members' performance
(note the larger benefit for the \cib{} model than for the \ciblsx{} model in 
Fig.~\ref{fig:mela:rel}). A reduction in between-split variation is observed,
when tuning the ensemble weights. Whether NLL or Brier score is optimized during
training influences test performance only slightly for this data set. Especially
for AUC, optimizing NLL instead of Brier score yields slightly better results (compare
Fig.~\ref{fig:mela:perf} \textsf{A} and \textsf{B}).
\begin{figure}[!ht]
    \centering
    \includegraphics[width=0.97\textwidth]{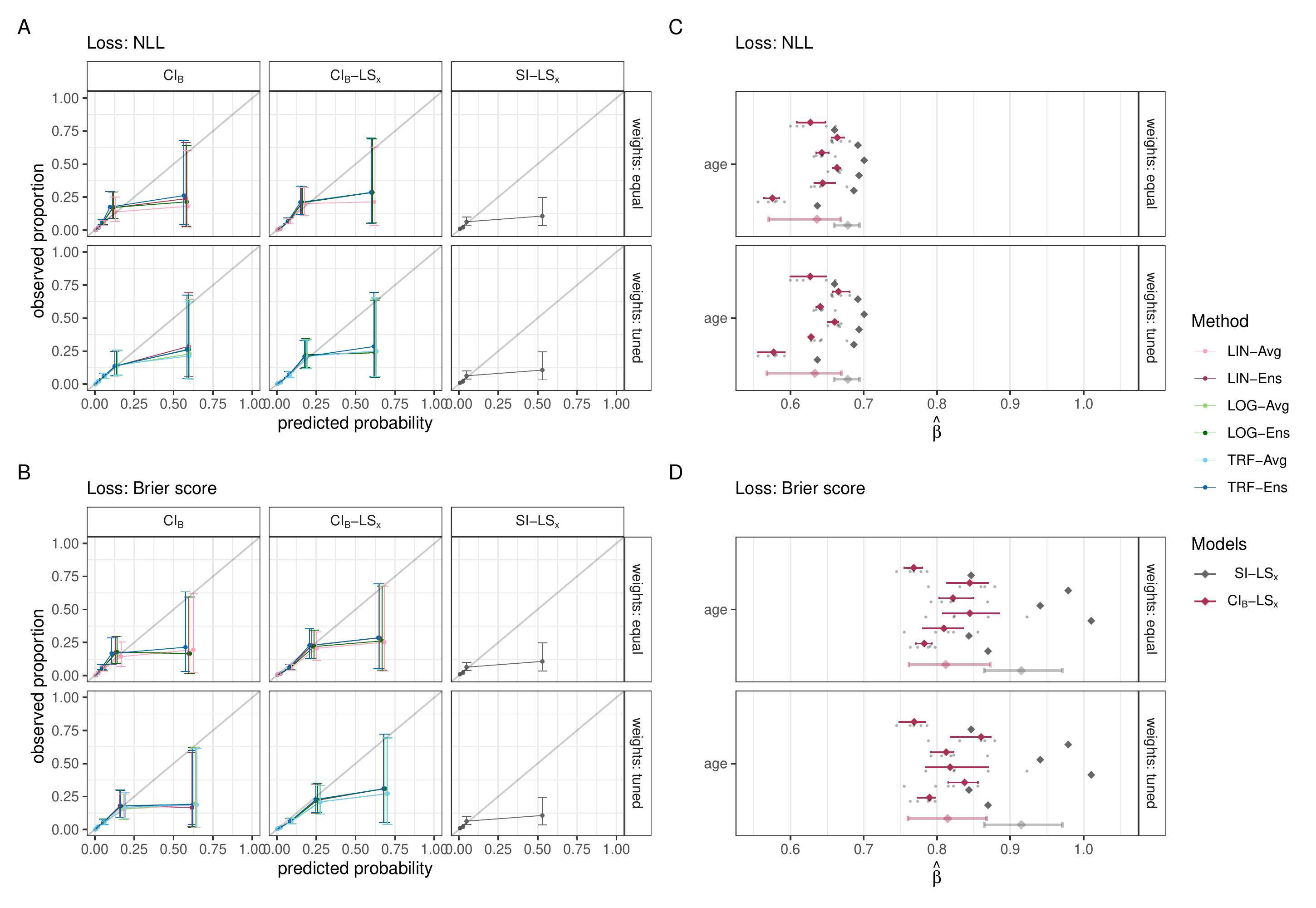}
    \caption{Calibration plots (\textsf{A} and \textsf{B}) and coefficient 
    estimates (\textsf{C} and \textsf{D}) for different models fitted on the 
    melanoma data set. Calibration with 95\% confidence intervals averaged 
    across splits are depicted for the classical linear (LIN-Ens),
    classical log-linear (LOG-Ens), and transformation (TRF-Ens) ensemble in 
    \textsf{A} and \textsf{B}. The predicted probabilities for a malign lesion
    are split at four empirical quantiles (0.5, 0.9, 0.99, 0.999) to compute
    the proportion of a malign lesion in the five resulting bins. 
    For the \silsx{}, the last two bins are merged.
    Panels \textsf{C} and \textsf{D} show log odds-ratios ($\hat\beta$) and 
    95\% bootstrap confidence intervals 
    for the standardized age predictor in the models \silsx{} and \ciblsx{}
    for each of the six random split. The average log odds-ratio across splits 
    is shown as transparent diamond along with a 95\% confidence interval.
    Individual log odds-ratios of the five ensemble members in each split
    are shown as transparent dots for the model \ciblsx{}.
    In the upper panels ensemble members are weighted equally and in the lower
    panels weights are tuned to minimize validation loss.}
    \label{fig:mela:calpl:lor}
\end{figure}

Calibrated predictions are hard to achieve when the outcome is highly imbalanced.
For predicted probabilities below $0.2$, all models seem to be well calibrated
(Fig.~\ref{fig:mela:calpl:lor} \textsf{A} and \textsf{B}). However, the \silsx{}
model over-predicts the probability for the rare outcome (malign lesion). Note
that the last bin includes only 0.1\% of observations and the model never predicts
probabilities of a malign lesion larger than 0.13. 
When including the image data, calibration improves somewhat and uncertainty
increases. Again, there is no pronounced difference between equal and tuned weights,
ensemble types, and ensembles and individual members, or NLL and Brier score loss. 
Calibration-in-the-large and calibration slope estimates are shown in
Appendix~\ref{app:figures}.

Lastly, since age is included as a tabular predictor, individual models and 
transformation ensembles thereof produce directly interpretable estimates for the 
log odds of a malign \vs a benign lesion upon increasing age by by one 
standard deviation (Fig.~\ref{fig:mela:calpl:lor} \textsf{C} and \textsf{D}). 
In all models, increasing age is associated with a higher risk of the lesion 
being malign. In transformation ensembles, the pooled estimate is simply the 
(weighted) average of the members' estimates. Note that the log odds-ratios
obtained by minimizing NLL and RPS agree in direction, but differ slightly in 
magnitude. In particular, the maximum RPS solution in \silsx{} is more uncertain
than the maximum likelihood solution (grey dots in Fig.~\ref{fig:mela:calpl:lor}
\textsf{C} \vs \textsf{D}).

\paragraph{UTKFace}
We fit six models of different flexibility and interpretability (see Table~\ref{tab:models}) 
to the UTKFace data in order to predict the conditional distribution of age categories 
given a person's face and/or sex and simulated tabular data. Fig.~\ref{fig:utkface:perf}
shows the average ensemble test error with bootstrap confidence intervals for four more
complex models (\ciblsx, \cib, \sicsblsx, \sicsb) and each ensembling method.
\begin{figure}[!ht]
    \centering
    \includegraphics[width=0.75\textwidth]{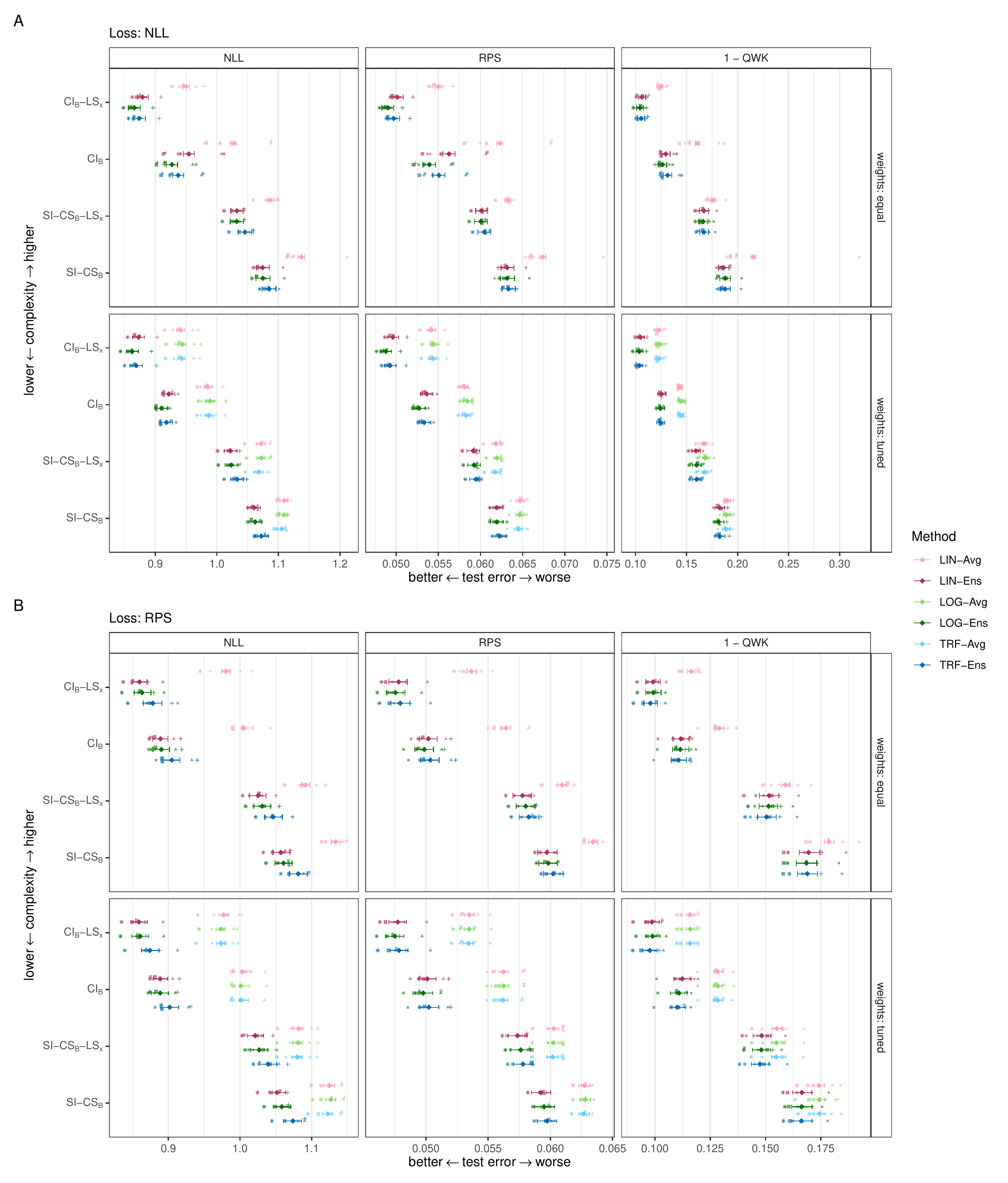}
    \caption{Performance estimates on the UTKFace data set. The classical linear 
    (LIN-Ens), classical log-linear (LOG-Ens), and transformation (TRF-Ens) 
    ensemble test error is shown for negative log-likelihood (NLL), ranked 
    probability score (RPS) and discrimination error measured by Cohen's 
    quadratic weighted kappa ($1-\operatorname{QWK}$). 
    The average ensemble test error and 95\% bootstrap confidence 
    intervals are depicted for six random splits (indicated by different symbols) 
    of the data. In the upper panels ensemble members are equally weighted 
    for constructing the ensemble and in the lower panels weights are tuned 
    to minimize validation loss. 
    In case of equal weights, the average coincides for all ensemble types (LIN-Avg).
    Models are fitted by minimizing NLL (\textsf{A}) or RPS (\textsf{B}).
    Note the different scales for \textsf{A} and \textsf{B}. The results for
    the \si{} and \silsx{} models are shown in Appendix~\ref{app:figures}.
    }
    \label{fig:utkface:perf}
\end{figure}
Test performance of the two simplest models (\si, \silsx) can be found in the 
Appendix~\ref{app:figures} along with the estimated coefficients for age and the 
ten simulated tabular predictors. Performance of the individual ensemble members 
and performance relative to the \silsx{} model are shown in Appendix~\ref{app:figures}.

Similar to the results found in \citet{kook2020ordinal} the most flexible model 
including both image and tabular data (\ciblsx) performs best across both proper
scores and discrimination metrics. Assuming proportional odds for the image term
does not seem to be appropriate, given the improved prediction performance when
loosening this assumption (\ciblsx{} \vs \sicsblsx{} and \cib{} \vs \sicsb).
All models including the image data perform considerably better than the benchmark 
model including only tabular data (\silsx, see Fig.~\ref{fig:utkface:rel}).

When comparing the different ensembling methods we again observe comparable
performance of the transformation ensemble to that of the classical ensembles 
for all test metrics. This difference is even more pronounced when 
ensemble weights are tuned to minimize validation loss (lower panels of 
Fig.~\ref{fig:utkface:perf}). As observed in the other data sets, 
tuning the ensemble weights reduces the variability in performance. 
As expected, models fitted by minimizing the RPS result in a lower test RPS.
However, the NLL of these models is also on par with or even better than when
optimizing the NLL. This provides some empirical evidence that optimizing the 
RPS may lead to more stable training for ordered outcomes, because it is bounded 
and global \citep{gneiting2005calibrated}.


\begin{figure}[!ht]
    \centering
    \includegraphics[width=0.90\textwidth]{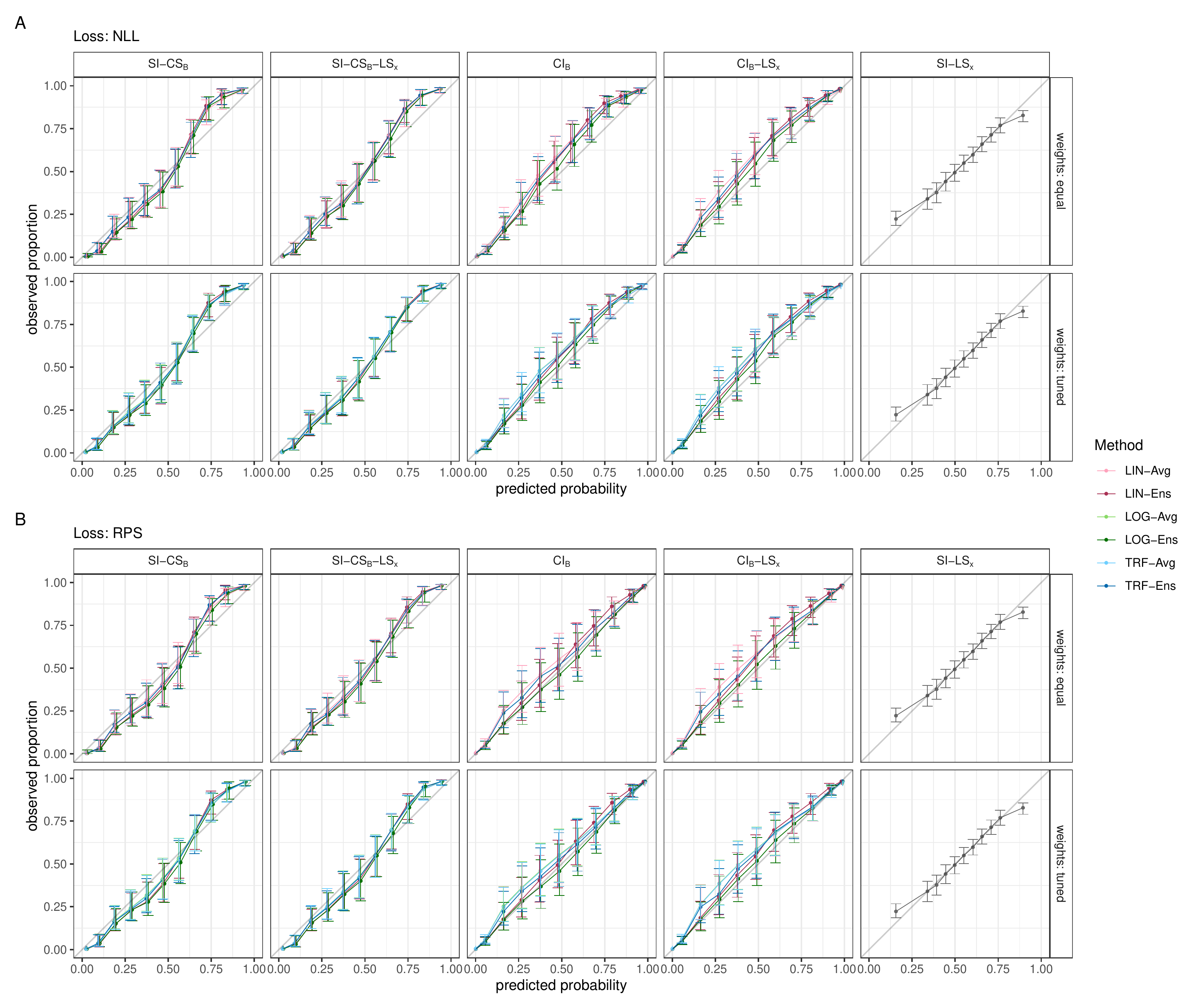}
    \caption{Calibration plots for different models fitted on the UTKFace data set.
    The average calibration across 6 random splits and 95\% 
    confidence intervals are depicted of the classical linear (LIN),
    classical log-linear (LOG), and transformation (TRF) ensemble.
    The predicted probabilities are split at the 0.5, 0.95 quantiles 
    and nine equidistant cut points in-between to calculate the observed 
    event rate in each bin.
    In the upper panels ensemble members are equally weighted 
    for constructing the ensemble and in the lower panels weights are tuned 
    to minimize validation loss. 
    Models are fitted by minimizing NLL (\textsf{A}) or RPS (\textsf{B}).}
    \label{fig:utkface:calpl}
\end{figure}
On the UTKFace data, models are fairly calibrated after training. The \silsx{}
model, including only tabular data, over-predicts large probabilities. Upon
including image data as complex intercepts, calibration again improves notably 
(\cib, \ciblsx), whereas under-prediction can be observed when modelling image
contributions as complex shifts (\sicsb, \sicsblsx). This can again be interpreted
as evidence that assuming proportional odds is unreasonable for this data set.
\cib{} and \ciblsx{} models are better calibrated when fitted with the RPS loss.
Again, no pronounced difference could be observed between the the two weighting 
schemes, or between ensembles and individual models. However, the log-linear ensemble 
produces the best-calibrated predictions in the \cib{} and \ciblsx{} models.
Calibration-in-the-large and calibration slope estimates are shown in 
Appendix~\ref{app:figures}.

For the ten simulated tabular predictors and sex the individual models with linear
shift terms and the corresponding transformation ensemble yield directly interpretable
log odds-ratios. In Figure~\ref{fig:utkface:or}, estimated log odds-ratios are depicted 
together with bootstrap confidence intervals. The \ciblsx{} models' estimates are 
close to the estimates of the \silsx{} models, whereas the estimates of the 
\sicsblsx{} model experience some shrinkage. The estimate for sex (female) 
changes sign upon inclusion of the image data, potentially due to collinearity between
learned image features and the tabular indicator variable. Again, the weighted average 
of the ensemble members' estimated coefficients is in fact the estimate in the 
transformation ensemble, which is not the case for any other type of ensemble 
discussed in this paper.

\paragraph{MNIST} 
Results for the MNIST data set can be found in Appendix~\ref{app:figures}.

\section{Summary and outlook} \label{sec:outlook}

Transformation ensembles bridge the gap between two common goals, prediction
performance and interpretability. Not only are they guaranteed to score better
than their ensemble members do on average, but in addition they preserve model 
structure and thus possess the same intrinsic interpretability as their members.
As a consequence, transformation ensembles allow to directly assess epistemic 
uncertainty in the intrinsically interpretable model parameters.
On multiple data sets we demonstrate that transformation ensembles improve both
probabilistic and discriminatory performance measures. Transformation ensembles
perform on par with classical ensembling approaches on all three datasets.
Thus transformation ensembles present a viable alternative to classical ensembling 
in terms of prediction and discrimination.

Preserving the intrinsic interpretability of its members is the most crucial benefit 
of transformation ensembles. Practitioners in fields where decisions are commonly 
based on multimodal and semi-structured data, such as medicine, require transparent,
well-performing and intrinsically interpretable models \citep{rudin2018stop}. 
As our results suggest, the increased flexibility of the model class when using 
classical ensemble techniques may often not be necessary. Instead, the more interpretable
transformation ensemble performs at least on par. In addition, transformation ensembles
simply pool interpretable model parameters in additive linear predictors using a (weighted) 
average. This does not only yield transformation ensembles with the same model structure 
and interpretability, but also allows to assess the epistemic uncertainty in linear shift
terms. This is in line with intuition but without theoretical 
justification when using any other type of ensemble. For complex shift and complex 
intercept terms which are not intrinsically interpretable, data analysts may still
use transformation ensembles and apply \emph{post hoc} explainability methods to those
model components \citep[Ch.~6]{molnar2020interpretable}.

Transparency requires clear communication of data and model uncertainty. Transformation
ensembles provide epistemic uncertainty estimates both in function space (transformation
function, CDF, or conditional quantiles) and parameter space (parameters of additive 
linear predictors). We use ensemble calibration to judge the quality of both aleatoric 
and epistemic uncertainty simultaneously.
In terms of calibration, no ensemble type led to a pronounced improvement over
individual models' average. For both data sets including tabular data (melanoma, UTKFace),
including the images alongside tabular data improved calibration. In clinical 
prediction modelling, where calibrated predictions are of high importance, it may
thus be advantageous to model image contributions or to aggregate re-calibrated
\citep[Ch.~15.3.5]{steyerberg2019} versions of the models.

Optimizing the ensemble weights on a hold out set, is a simple way to reduce 
variability between random splits of the data. If only little data is available
for training, a cross validation scheme may yield similar benefits when averaging
the weights over the cross validation folds.

In this article, we mainly focused on interpretability, prediction and calibration.
Apart from benefits in prediction performance and uncertainty quantification, deep 
ensembles have found to be robust when evaluated out of distribution, \eg under 
distributional shifts \citep[for a discussion, see][]{abe2022deep}.
We leave for future research how transformation ensembles compare to classical linear
ensembling in terms of robustness towards distributional shifts and related issues in
epistemic uncertainty estimation. Further theoretical directions include whether minimax
optimality of transformation ensembles for binary outcomes extends to ordered and
continuous outcomes.

\paragraph{Acknowledgements}
We thank Jeffrey Adams, Oliver D\"urr, Lisa Herzog, David R\"ugamer and Kelly Reeve 
for their valuable comments on the manuscript. The research of LK and BS was supported 
by Novartis Research Foundation (FreeNovation~2019) and by the Swiss National Science
Foundation (grant no. S-86013-01-01 and S-42344-04-01). TH was supported by the
Swiss National Science Foundation (SNF) under the project ``A Lego System for
Transformation Inference'' (grant no. 200021\_184603).


\vskip 0.2in
\bibliographystyle{plainnat}
\bibliography{bibliography} 


\appendix
\renewcommand{\thesection}{\Alph{section}}
\counterwithin{figure}{section}
\renewcommand\thefigure{\thesection\arabic{figure}}
\counterwithin{table}{section}
\renewcommand\thetable{\thesection\arabic{table}}

\section{Notation} \label{app:notation}

Random variables are written in uppercase italic, $\rY$, and realizations thereof
in lowercase, $\ry$. In case, the random variable is vector-valued, we write it and
its realizations in bold, \eg $\rX$, $\rx$. We denote probability measures by
$P$, PDFs by $f$ and CDFs by $F$. 
Ensemble members are denoted by $\pmemb$, or $\dmemb$, $\be = 1, \dots, \Be$, when
referring to the CDF or PDF, respectively. An ensemble is some average of these 
members and its CDF and PDF are denoted by $\pens$, $\dens$, respectively. 
A superscript then denotes the type of ensemble, \eg $\pens^c$ the classical 
linear ensemble. If, for two functions $f,g : A \to B$, $f(a) \leq g(a)$
for all $a \in A$, we omit the argument and write $f < g$.

\section{Computational details} \label{app:experiments}

The code for reproducing all experiments can be found on GitHub 
\url{https://github.com/LucasKook/interpretable-deep-ensembels}.
All components of the models used in this work (see Table~\ref{tab:models} for an 
overview) were controlled by neural networks and fitted jointly by minimizing the
negative log-likelihood or ranked probability score via stochastic gradient descent 
using the Adam optimizer \citep{Kingma2015adam}. The unconditional model (\si)
and the model containing tabular predictors only (\silsx) were fitted using the 
function \texttt{tram::Polr} \citep{pkg:tram} in case of ordinal responses and
\texttt{stats::glm} in case of binary responses. Both models were also 
implemented as \textsc{ontram}s to minimize the RPS. Factor variables were dummy 
encoded and continuous predictors were standardized to zero mean and unit variance.
All models were implemented in \textsf{R} \citep[version 4.1.1]{pkg:base}. 
Semi-structured data was modeled using Keras \citep[version 2.7.0]{pkg:keras} based 
on the TensorFlow backend \citep[version 2.7.0]{pkg:tensorflow} and trained on a GPU.

Simple intercept terms (and the \si{} model) were estimated using a 
fully-connected single-layer neural network with linear activation function, no 
bias term, and $K-1$ output nodes. Linear shift terms (\lsx) were estimated 
similarly but with a single output node. Both, simple intercept and linear shift 
terms were initialized using the parameters estimated by either a proportional odds 
logistic regression model (POLR) or a generalized linear model (GLM).

Terms depending on the image modality (\cib, \csb) were modeled using a 2D 
convolutional neural network (CNN) for all data sets. The last layer was common to 
all CNNs and consisted of either $K-1$ output nodes for complex intercept terms or 
a single output node for complex shift terms. The identity was used as activation 
(ReLU non-linearity was used in all other layers) and no bias term was included.
In case of intercept terms, the raw intercepts $\gamma \in \RR^{K-1}$ were further 
transformed to the intercepts $\parm$ via a cumulative softplus transformation
\begin{align*}
    \parm = \left(\gamma_1, \gamma_1 + \log(1 + \exp(\gamma_2)), \dots, \gamma_1 + 
    \sum_{k=2}^{K-1} \log(1 + \exp(\gamma_k))\right),
\end{align*}
to ensure monotonicity. The CNNs used to control complex model terms (\cib, \csb)
were initialized as Glorot uniform, and zero bias.
To prevent the models from overfitting, weights from the epoch with the smallest
validation loss were used as the final model. For all data sets, images were 
normalized to the unit interval via $(y - \min(y)) / (\max(y) - \min(y))$. 
No further pre-processing steps were undertaken.

For all data sets, the data was randomly split six times into a training, validation 
and a test set. Five models trained and evaluated on the same split 
built an ensemble. Hyperparameters, such as the number of training epochs or the 
weights for constructing the weighted ensemble, were tuned on the validation set.
In applications it may be advised to use a nested cross-validation scheme instead
of tuning on a single validation set. Model architectures and hyperparameters varied 
across data sets and will be described for each data set in the following.

\paragraph{MNIST}
To predict the nominal outcome a CNN with three convolutional stacks consisting of 
a convolutional layer and a max-pooling layer (window size $2 \times 2$ pixels, stride 
width $2$) was used. The number of filters used in each stack were $32$, $64$, $64$.
Filter size was set to $3 \times 3$ pixels. The fully-connected part comprised 
a layer with $100$ units followed by the output layer (see Fig.~\ref{fig:architectures}).

All 30 models (six ensembles consisting of five members each) were trained on 80\% and 
validated on 10\% of the data. Ensembles were validated on 10\% of the data. A learning 
rate of $10^{-5}$ and a batch size of $512$ were used for model training. 
\begin{figure}[!ht]
    \centering
    \resizebox{\linewidth}{!}{%
    \begin{tikzpicture}
    \node (A) {\small \bf MNIST};
    \node[right=3.5cm of A] (B) {\small \bf Melanoma};
    \node[right=3.5cm of B] (C) {\small \bf UTKFace};
    \node[below=0cm of A.south east] (mnist)
    {\includegraphics[width=0.3\textwidth]{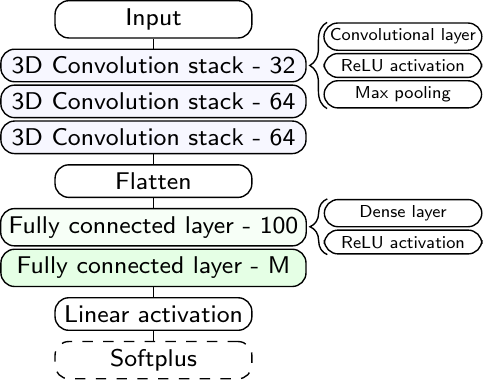}};
    \node[below=0cm of B.south east] (mela)
    {\includegraphics[width=0.3\textwidth]{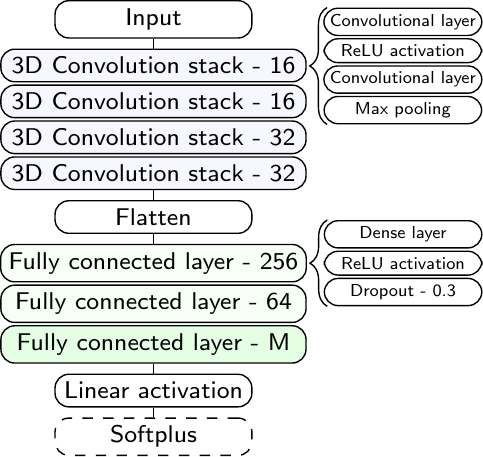}};
    \node[below=0cm of C.south east] (face)
    {\includegraphics[width=0.3\textwidth]{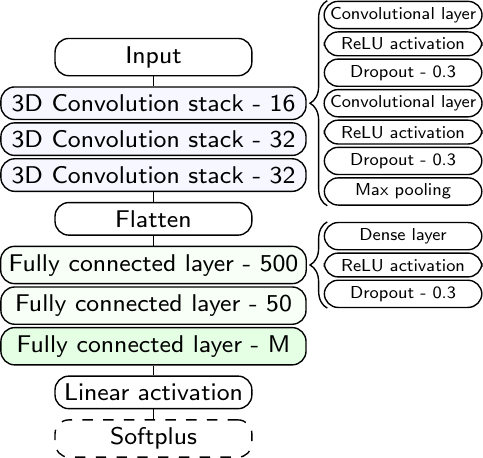}};
    \end{tikzpicture}}
    \caption{CNN architecture used for the MNIST, melanoma, and UTKFace data.}
    \label{fig:architectures}
\end{figure}

\paragraph{Melanoma} 
\cib{} and \csb{} terms were modeled using a CNN with two convolutional stacks 
with $16$ filters ($3 \times 3$ pixels) and two convolution stacks with $32$
filters ($3 \times 3$ pixels). Each convolutional layer was followed by a dropout 
layer (rate $0.3$) and a max-pooling layer (window size $2 \times 2$ pixels, stride 
width $2$). For the subsequent fully-connected part, two dense layers with $256$ 
and $64$ units and two dropout layers (rate $0.3$) were used (see 
Fig.~\ref{fig:architectures}).

Models were trained on 60\%, validated on 20\% and evaluated on 20\% of the data.
CNNs were trained using a learning rate of $10^{-4}$ and a batch size of 64. 

\paragraph{UTKFace}
A similar CNN architecture as in \cite{kook2020ordinal} was used to incorporate the 
image modality (Fig.~\ref{fig:architectures}). The architecture was originally 
inspired by \cite{simonyan2014very}. The model is given by three convolutional 
stacks of which each consists of two convolutional layers, two dropout layers 
(rate $0.3$) and a max-pooling layer (window size $2 \times 2$ pixels, 
stride width $2$). In the first two convolutional stacks $32$ filters were used 
followed by two stacks using $64$ filters ($3 \times 3$ pixels). After flatting,
the output entered the fully connected part comprised of two dense layers separated
by a dropout layer (rate $0.3$). 

All models were trained on 60\%, validated on 20\% and evaluated on 20\% of the data.
CNNs were trained using a learning rate of $10^{-3}$ and a batch size of 32.

\paragraph{Evaluation metrics}
Classification metrics depend on the (top-label) classification, which we define as 
the mode of the predicted conditional density, \ie
\begin{align*}
    c = \argmax_k f(\ry_k \given \calD),
\end{align*}
from which we then compute the accuracy
\begin{align*}
  \operatorname{ACC} = \frac{1}{n} \sum_{i = 1}^n \delta_{\ry_i}^{c_i},
  \mbox{ where } \delta_a^b = \begin{cases} 1 & \mbox{if } a = b \\ 0 & 
  \mbox{otherwise}, \end{cases}
\end{align*}
and Cohen's quadratic weighted $\kappa$,
\begin{align*} 
\kappa_w = \frac{\sum_{i,j}w_{ij}o_{ij} - \sum_{i,j}w_{ij}e_{ij}}{1 -
    \sum_{i,j}w_{ij}e_{ij}}.
\end{align*}
Here, $e$ and $o$ denote the expected and observed number of agreements in row $i$ 
and column $j$ of a confusion matrix, respectively. In addition, the weights 
$w_{ij} = \frac{\lvert i - j \rvert^q}{(K-1)^p}$ control the penalization for 
predictions farther away from the observed class via $q$.

The AUC is a discrimination metric, defined as
\begin{align*}
    \operatorname{AUC} &= \max\{\operatorname{PI}, 1 - \operatorname{PI}\}, 
    \mbox{ where } \\
    \operatorname{PI} &= \Prob(\rY_1 < \rY_2 \given \calD_1, \calD_2) +
    \frac{1}{2} \Prob(\rY_1 = \rY_2 \given \calD_1, \calD_2).
\end{align*}
Here, PI denotes the probabilistic index and $\rY_1 \given \calD_1$ and
$\rY_2 \given \calD_2$ are independent and distributed according to $F_{\rY \given \calD}$ 
\citep{thas2012probabilistic}.

Lastly, we use calibration plots and compute two calibration metrics, calibration
in the large and calibration slope \citep{steyerberg2019}. Calibration plots are 
produced using the \pkg{caret} add-on package for \proglang{R} \citep{pkg:caret}.
Confidence intervals are obtained via standard binomial tests and averaged over 
splits. The calibration slope $b_1$ is obtained by regressing the indicator for a 
class on the predicted log odds for that class,
\begin{align*}
    \logit(\Ex[\I(\rY > \ry_k) \given \hat r_k]) = b_0 + b_1\hat r_k,
\end{align*}
where $\hat r_k = \logit(\hat\Prob(\rY > \ry_k \given \calD))$. Similarly, calibration 
in the large $a$, is the intercept in the model offsetting the predicted log odds
$\hat r_k$,
\begin{align*}
    \logit(\Ex[\I(\rY > \ry_k) \given \hat r_k]) = a + 1\hat r_k.
\end{align*}

\paragraph{Bootstrap confidence intervals}
For each metric, we report bootstrap confidence intervals over $b = 1, \dots, B = 1'000$
bootstrap samples of $i = 1, \dots, n_{\rm test}$ test observations, for each of 
the $s = 1, \dots, S = 6$ splits and $\be = 1, \dots, \Be = 5$ members, by taking the
2.5th and 97.5th empirical quantile of the bootstrap metric. For the three ensemble types,
there is no step of averaging over the $\Be$ members. In case of relative performance,
the fixed \silsx{} performance per split was subtracted instead of the bootstrap 
metric. To construct confidence intervals for the coefficient estimates, the $\Be$ 
member estimates were bootstrapped.

\section{Proofs and additional results} \label{app:theory}

\paragraph{Proof of Proposition~\ref{prop:interpretable}}
We re-state Prop.~\ref{prop:interpretable}. For the proof, we exploit the fact 
that in transformation models, $\pZ$ has a log-concave density.
\begin{proposition*}
Let $\pZ \circ \h_1, \dots, \pZ \circ \h_\Be$ be transformation model CDFs with
$w_1, \dots, w_\Be$ non-negative weights summing to one. Let $\ell: \calP \to \RR$
denote the log-likelihood induced by the transformation model $\pZ \circ \h$. Then,
$-\ell\left(\pZ(\sum_\be w_\be \hb)\right) \leq 
- \sum w_\be \ell\left(\pZ\circ\hb \right)$.
\end{proposition*}
\begin{proof}
The log-density of the transformation ensemble is given by
\begin{align*}
    \log \dens^t(\ry) = \log \frac{\dd}{\dd\ry} \pZ\left(\sum_\be w_\be \hb(\ry)\right) = 
    \log\left[\dZ\left(\sum_\be w_\be \hb(\ry)\right)\sum_\be w_\be \hb'(\ry)\right].
\end{align*}
Now, the result follows by applying Jensen's inequality for both terms,
\begin{align*}
  - \log \dZ \left(\sum_\be w_\be \hb\right) - \log \left( \sum_\be w_\be \hb'\right)
  \leq - \sum_\be w_\be \log \left(\dZ \circ \hb\right) - \sum_\be w_\be \log \hb',
\end{align*}
once using log-concavity of $\dZ$ and once using concavity of the log.

In case of interval-censored responses (see Section~\ref{sec:bg}), 
$y \in (\ubar\ry, \bar\ry]$ the likelihood contribution,
\begin{align*}
    - \log\left[\pZ\left(\sum_\be w_\be \hb(\bar\ry)\right) - 
    \pZ\left(\sum_\be w_\be \hb(\ubar\ry)\right)\right],
\end{align*}
is convex in $\hb$ because the transformation function is monotone non-decreasing and 
$\ubar\ry < \bar\ry$. The desired claim follows again via Jensen's inequality.
\end{proof}
Similar results for the Brier score and RPS (a sum of Brier scores) do not hold 
for general $\pZ$, because log-concavity of $\pZ$ is not a strong enough 
condition to ensure convexity of $L \circ \pZ$. 

Next, we briefly state and prove propositions for better-than-average prediction
performance of the two variants of the log-linear ensemble (PDF and CDF pooling).
\begin{proposition}\label{prop:logensdens}
Let $\dmem_1, \dots, \dmem_\Be$ be conditional densities with non-negative weights
$w_1, \dots, w_\Be$ summing to one. Let $\dens = c \exp\left(\sum_\be w_\be \dmemb\right)$,
where $c^{-1} := \int_{-\infty}^{+\infty} \exp(\sum_\be w_\be \log \dmemb)$.
Then, $\NLL(\dens) \leq \sum_\be w_\be \NLL(\dmemb)$.
\end{proposition}
\begin{proof}
The ensemble NLL is given by $\NLL(\dens) = - \log c - \sum_\be w_\be \log \dmemb$,
hence we need to show $c \geq 1$. Indeed, we have
\begin{align*}
   c^{-1} = \int_{-\infty}^{+\infty} \exp\left(\sum_\be w_\be \log \dmemb\right)
   \leq \int_{-\infty}^{+\infty} \sum_\be w_\be \dmemb = 1,
\end{align*}
via Jensen's inequality.
\end{proof}

\begin{proposition}\label{prop:logenscum}
Let $\pmem_1, \dots, \pmem_\Be$ be conditional CDFs with non-negative weights
$w_1, \dots, w_\Be$ summing to one. Let $\pens = \exp(\sum_\be w_\be \pmemb)$.
Then, $\NLL(\dens) \leq \sum w_\be \NLL(\dmemb)$, where 
$\dens = \frac{\dd}{\dd\ry} \pens = \exp(\sum_\be \log \pmemb)
(\sum_\be w_\be \dmemb/\pmemb)$.
\end{proposition}
\begin{proof}
The ensemble NLL is given by
\begin{align*}
    \NLL(\dens) &= - \sum_\be w_\be \log \pmemb - 
        \log\left(\sum_\be w_\be \dmemb/\pmemb\right) \\
    &\leq - \sum_\be w_\be \log \pmemb - 
        \sum_\be w_\be \log \dmemb + \sum_\be w_\be \log \pmemb \\
    &= - \sum_\be w_\be \log \dmemb = \sum_\be w_\be \NLL(\dmemb),
\end{align*}
by Jensen's inequality.
\end{proof}

\begin{remark}
Propositions~\ref{prop:logensdens} and~\ref{prop:logenscum} illustrate 
that both log-linear ensembles of PDFs and CDFs produce better than average
predictions. Normally densities are aggregated in log-linear ensembles. However,
since we propose a CDF-based ensemble, the CDF version of the log-linear
ensemble is a more direct comparator to transformation ensembles in our 
experiments.
\end{remark}

Next, we present more results on minimax optimality of quasi-arithmetic pooling
and transformation ensembles.

\paragraph{Proof of Corollary~\ref{thm:minimax}}
First, we give some intuition.
The approach by \citet{neyman2021proper} can be inverted to find a (regular) proper
scoring rule corresponding to a given pooling method. The transformation ensemble
in general uses quasi-arithmetic pooling with $g = \pZ^{-1}$ on the CDF scale. Now,
instead of finding the minimax optimal pooling for a score, we take the reverse route 
and construct
\begin{align} \label{eq:quantint}
    G(p) = \int_0^p \pZ^{-1}(u) \dd u.
\end{align}
Eq.~\eqref{eq:quantint} is an integrated quantile function and has close connections
to the expectation of $\rZ$ and trimmed expectations in general. This can be seen in 
$\Ex[\rZ] = \Ex[\pZ^{-1}(U)] = \int_0^1 \pZ^{-1}(u) \dd u$, where $U$ is standard 
uniform. Before we prove Corollary~\ref{thm:minimax}, we restate Theorem~4.1 from
\citet{neyman2021proper}.
\begin{theorem}[Minimax optimality, Theorem~4.1 in \citet{neyman2021proper}]
\label{thm:neyman}
Let $p_1, \dots, p_\Be$ be probability densities for a nominal outcome, 
\ie $p_\be(\ry_k) = \Prob_\be(\rY = \ry_k)$, with non-negative weights 
$w_1, \dots, w_\Be$ summing to one. Then
\begin{align*}
    \max_\ry s(p, \ry) - \sum_{\be = 1}^\Be w_\be s(p_\be, \ry)
\end{align*}
is minimized by $p = g^{-1}(\sum_\be w_\be g(p_\be))$, where $s$ is a
proper score and $g$ is a subgradient of the expected score.
\end{theorem}
We now re-state and prove Corollary~\ref{thm:minimax}.
\begin{corollary*}
Let $p_1, \dots, p_\Be$ be predicted probabilities for success in a binary outcome, 
\ie $p_\be = \Prob_\be(\rY = 1 \given \calD)$ and $w_1, \dots, w_\Be$ be non-negative
weights summing to one. Then $\bar p_\Be^t = \expit(\sum_\be w_\be\logit(p_\be))$ 
minimizes
\begin{align*}
    \max_\ry \NLL(p, \ry) - \sum_{\be = 1}^\Be w_\be \NLL(p_\be, \ry).
\end{align*}
\end{corollary*}
\begin{proof}
For transformation ensembles with $\pZ = \expit$ and binary outcomes, we have
for each member $p \in \{p_1, \dots, p_\Be\}$,
\begin{align*}
    G(p) = \int_0^p \pZ^{-1}(u) \dd u = p \log p + (1 - p) \log (1 - p) + C 
        = - \Ex[\NLL(p, Y)] + C.
\end{align*}
The result then follows from Theorem~\ref{thm:neyman}.
\end{proof}
Thus we find that transformation ensembling with $\pZ = \expit$ is minimax optimal 
for binary outcomes in terms of NLL. Note that the restriction to binary outcomes 
is necessary, because \citet{neyman2021proper} ensemble densities of nominal outcomes, 
for which the CDF is ill-defined.
We aggregate predictions on the CDF scale, because this ensures interpretability of 
the transformation ensemble. In turn, the two approaches can only be equivalent for 
binary outcomes, for which the distribution is fully characterized by 
$\Prob(\rY \leq \ry_0) = \Prob(\rY = \ry_0) = p$ when defining the ``order'' of the 
binary outcome appropriately.

\paragraph{Minimax optimal pooling for the RPS}
\citet{neyman2021proper} present results solely for nominal outcomes, for which the
CDF is not well-defined. However, we can still extend their results by considering
linear pooling, for which the aggregation scale (CDF or PDF) does not matter. Then,
we arrive at the following result.
\begin{corollary}[Classical linear ensembling is minimax optimal i.t.o. RPS]
\label{thm:minmaxrps}
Let $\pmem_1, \dots, \pmem_\Be$ be predicted conditional CDFs for an ordinal outcome
$\rY$ with sample space $\{\ry_1 < \dots < \ry_K\}$, 
\ie $\pmemb(\ry_k) = \Prob_\be(\rY \leq \ry_k \given \rx)$,
with non-negative weights $w_1, \dots, w_\Be$ summing to one. Then
\begin{align*}
    \max_\pmem \RPS(\pmem, \ry_k) - \sum_{\be = 1}^\Be w_\be \RPS(\pmemb, \ry_k)
\end{align*}
is minimized by $\pmem = \sum_\be w_\be \pmemb = \pens^c$.
\end{corollary}
\begin{proof} 
The expected RPS is given by
\begin{align*}
  \Ex[\RPS(F, Y)] = f(y_1)\RPS(F, y_1) + f(y_2)\RPS(F, y_2) + \dots +
    f(y_K) \RPS(F, y_K),
\end{align*}
where each RPS term can be simplified to
\begin{align*}
  (K - 1) \RPS(F, y_k) 
  & = \sum_{l=1}^K \left(F(y_l) - \I(y_k \leq y_l) \right)^2 \\
   & = \sum_{l=1}^K F^2(y_l) - 2F(y_l)\I(y_k \leq y_l) + \I(y_k \leq y_l)^2 \\
   & = \sum_{l=1}^K F^2(y_l) - 2F(y_l)\I(y_k \leq y_l) + \I(y_k \leq y_l) \\
   & = \sum_{l=1}^K F^2(y_l) - 
    \sum_{l=1}^K (2F(y_l) - 1) \I(y_k \leq y_l) \\
  &= \sum_{l=1}^K F^2(y_l) - \sum_{l=k}^K 2 F(y_l) - (K - k).
\end{align*}
We can thus write the expected RPS as
\begin{align*}
    (K - 1) \Ex[\RPS(F, Y)] &= (K - 1) \sum_{k=1}^K f(y_k) \RPS(F, y_k) \\
    &= \sum_{l=1}^K F^2(y_l) - \sum_{k=1}^K \left[f(y_k) \left(\sum_{l = k}^K 2 F(y_l)
    - (K - k)\right)\right].
\end{align*}
It is now sufficient to show that the derivative of the expected RPS w.r.t.
$f(y_1), \dots, f(y_K)$ is affine in $f(y_k)$ to obtain the result.
Now, the first term is independent of $f(y_k)$. The second sum is indeed 
proportional to $f(y_k)$. Now linear pooling being minimax optimal in terms 
of RPS follows from Theorem~\ref{thm:neyman}.
\end{proof}


\begin{landscape}
\section{Scoring rules} \label{app:scoringrules}

Here, we collect the (proper) scoring rules used in this article with their exact 
definition. We do so to prevent confusion with varying nomenclature in other parts 
of the literature. Improper scores and other evaluation metrics we use are 
listed in Appendix~\ref{app:experiments}.

\begin{table}[!ht]
    \centering
    \caption{List of scoring rules, definitions, expectations, divergences, minimax
    optimal quasi-arithmetic pooling schemes. The table is largely inspired by 
    \citet{brocker2009reliability}. Results on QA pooling for Brier's quadratic 
    score and the log-score for nominal outcomes can be found in 
    \citet{neyman2021proper}. All others are derived in Appendix~\ref{app:theory}.
    Binary outcomes: $\rY \sim \operatorname{Bernoulli}(p)$.
    Nominal outcomes: $\rY \sim \operatorname{Discrete}(p_1, \dots, p_K)$ with density
    $f(\ry_k) = \Prob(\rY = \ry_k)$. Ordinal outcomes: Same as nominal, but the CDF 
    $\pmem(\ry_k) = \Prob(\rY \leq \ry_k)$ is well-defined.}
    \label{tab:scores}
    \resizebox{\linewidth}{!}{%
    \renewcommand{\arraystretch}{3}
    \begin{tabular}{cccccc}
    \toprule
     \bf Name & \bf Outcome & \bf Definition & \bf Entropy & \bf Divergence & 
     \bf QA pooling 
     \\ \midrule
     Brier score 
     & Binary 
     & $(y - p)^2$ 
     & $p(1 - p)$ 
     & $(p - q)^2$
     & $\sum_\be w_\be p_\be$
     \\
     
     Brier's quadratic score 
     & Nominal
     & $\sum_l f^2(\ry_l) - 2 f(\ry_k)$ 
     & $\sum_l p_l^2$
     & $\sum_l (p_l - q_l)^2$
     & $\sum_\be w_\be \dmemb(\ry_k)$
     \\
     
     Ranked probability score 
     & Ordinal 
     & $\frac{1}{K-1} \sum_l (F(\ry_l) - \I(\ry_k \leq \ry_l))^2$ 
     & $\frac{1}{K-1} \sum_l F(\ry_l) (1 - F(\ry_l))$
     & $\frac{1}{K-1} \sum_l (F(\ry_l) - G(\ry_l))^2$
     & $\sum_\be w_\be \pmemb(\ry_k)$
     \\
     
     log-score, ignorance, NLL
     & Binary
     & $- y \log p - (1 - y) \log 1 - p$ 
     & $- p \log(p) - (1 - p) \log(1 - p)$
     & $- p \log\frac{p}{q} - (1 - p) \log\frac{1 - p}{1 - q}$
     & $\expit\left(\sum_\be w_\be \logit p_\be \right)$
     \\
     
     log-score, ignorance, NLL
     & Nominal
     & $- \log \dmem(\ry_k)$ 
     & $- \sum_l \dmem(\ry_l) \log(\dmem(\ry_l))$
     & $- \sum_l \dmem(\ry_l) \log\left(\frac{\g(\ry_l)}{\dmem(\ry_l)}\right)$
     & $\exp\left(\sum_\be w_\be \log \dmemb(\ry_k) \right)$
     \\
     \bottomrule
    \end{tabular}}
\end{table}
\end{landscape}

\section{Additional empirical results} \label{app:figures}

In this appendix, we provide further results for each data set. We show performance 
of the individual ensemble members as gray dots, in order to better judge the
improvement ensembling yields over single members. In addition, we show performance
relative to the \silsx{} model (by taking differences within splits) in order to
eliminate between-split variation (for the melanoma and UTKFace data). Lastly, we
show calibration-in-the-large (CITL) and calibration slope (C slope) and coefficient
estimates for the additive linear predictors (when present in the model).

\subsection{MNIST}

Based on the MNIST image data we fit a complex intercept model (see Table~\ref{tab:models}) 
in order to predict the conditional probabilities of single digits from 0 to 9. The test
performance of the different ensemble methods and the average performance of individual
ensemble members are shown in Fig.~\ref{fig:mnist:perf}. 
\begin{figure}[!ht]
    \centering
    \includegraphics[width=\textwidth]{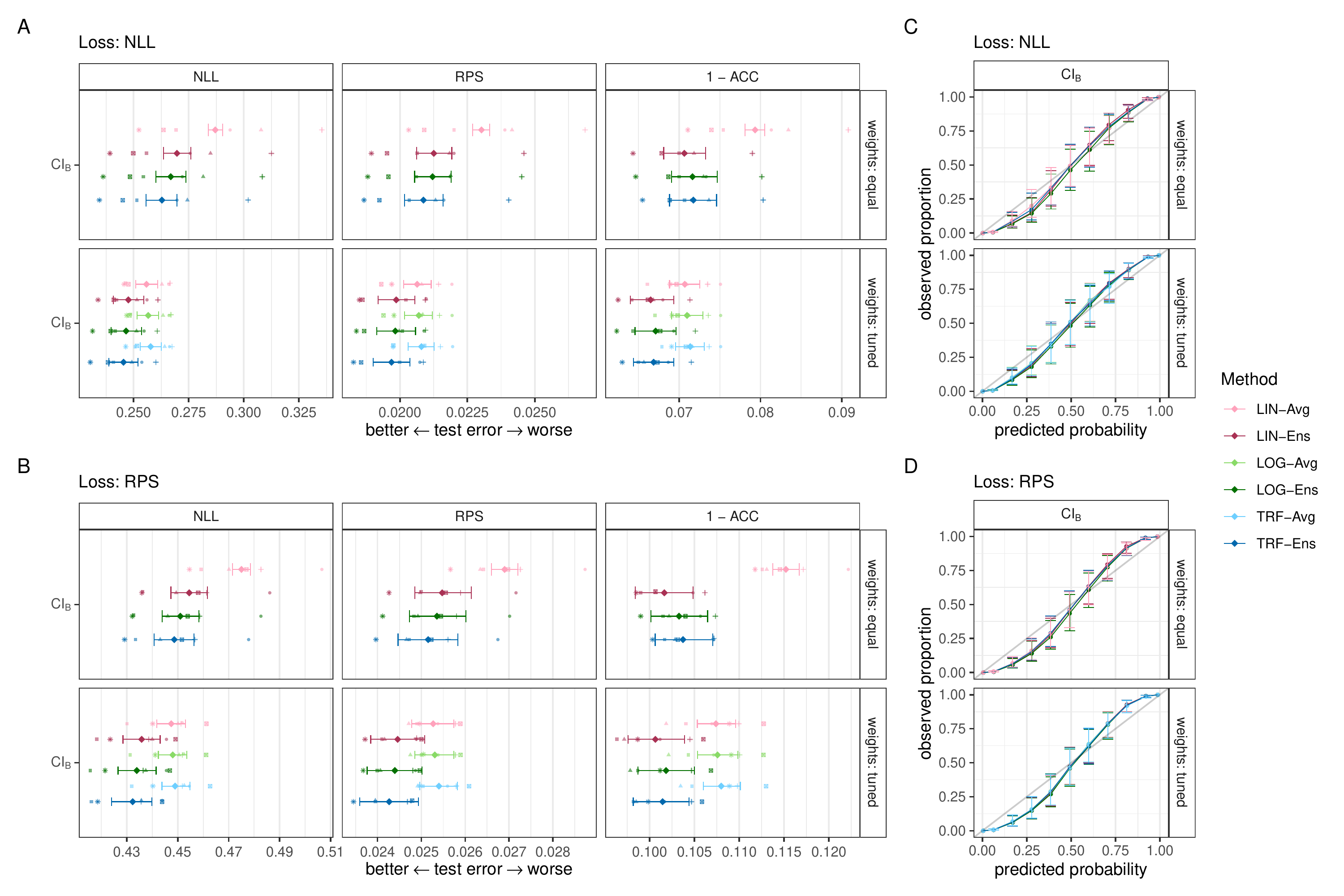}
    \caption{Performance estimates (\textsf{A}, \textsf{B}) and 
    calibration plots (\textsf{C}, \textsf{D}) for the \cib{} model fitted
    on the MNIST data set. The classical linear (LIN-Ens), classical log-linear 
    (LOG-Ens), and transformation (TRF-Ens) ensemble test error is shown for 
    negative log-likelihood (NLL), ranked probability score (RPS) and classification 
    error ($1-\operatorname{ACC}$) in \textsf{A} and \textsf{B} along with the
    average (AVG) performance measures of the individual members. The average 
    ensemble test error and 95\% bootstrap confidence intervals are depicted 
    for six random splits of the data (indicated by different symbols).
    \textsf{C} and \textsf{D} show calibration and 95\% confidence intervals 
    averaged across outcome classes and splits. The predicted probabilities were 
    split at the 0.5, 0.95 quantiles and nine equidistant cut points in-between to 
    calculate the observed event rate in each bin (for details, see
    Appendix~\ref{app:experiments}).
    In the upper panels ensemble members were equally weighted for constructing the
    ensemble and in the lower panels weights were tuned to minimize validation loss.
    In case of equal weights, the average coincides for all ensemble types (LIN-Avg).
    Models were trained by minimizing NLL (\textsf{A}, \textsf{C}) or RPS (\textsf{B},
    \textsf{D}). Note the different scales for \textsf{A} and \textsf{B}.}
    \label{fig:mnist:perf}
\end{figure}

When additionally tuning the ensemble weights on the 
validation set, test prediction and discrimination performance improve further. 
The between-split variation in performance is lower for the weighted ensembles than 
in ensembles where each member has the same weight, especially when the models were
trained by minimizing the NLL (compare upper and lower panels of Fig.~\ref{fig:mnist:perf}).
For each metric, transformation ensembles perform at least on-par with classical
(LIN-Ens and LOG-Ens) ensemble approaches. In this complex intercept model without 
interpretable additive components, no interpretability is gained when using the 
transformation instead of the linear or log-linear ensemble.

Note that even though the outcome is
nominal, the arbitrarily chosen ordering of the outcome does not influence performance
when training with the NLL loss. However, when training the models using the RPS loss 
which explicitly uses the ordering of the outcome performance is worse in all proper 
scores and discrimination metrics compared to the models trained by minimizing the NLL
(\cf Fig.~\ref{fig:mnist:perf}\textsf{A} and \textsf{B}).

The calibration plots (Fig.~\ref{fig:mnist:perf}\textsf{B} and \textsf{D}) show a
tendency for overconfident predictions. Calibration was slightly better when 
minimizing NLL instead of RPS during training. No pronounced difference in
calibration was observed between single models and ensembles, between types of 
ensembles or between equal and tuned ensemble weights. Calibration-in-the-large
and calibration slope estimates are shown in Appendix~\ref{app:figures}.

\paragraph{Additional results for the MNIST data set}
Fig.~\ref{fig:mnist:indiv} shows the performance of all $5 \times 6$ individual 
models. Conclusions from the main text remain unaltered. Calibration metrics are
summarized in Fig.~\ref{fig:mnist:cal:indiv}, and again show the slight 
miscalibration of all models, with worse calibration when using the RPS loss.

\begin{figure}[!ht]
    \centering
    \includegraphics[width=0.7\textwidth]{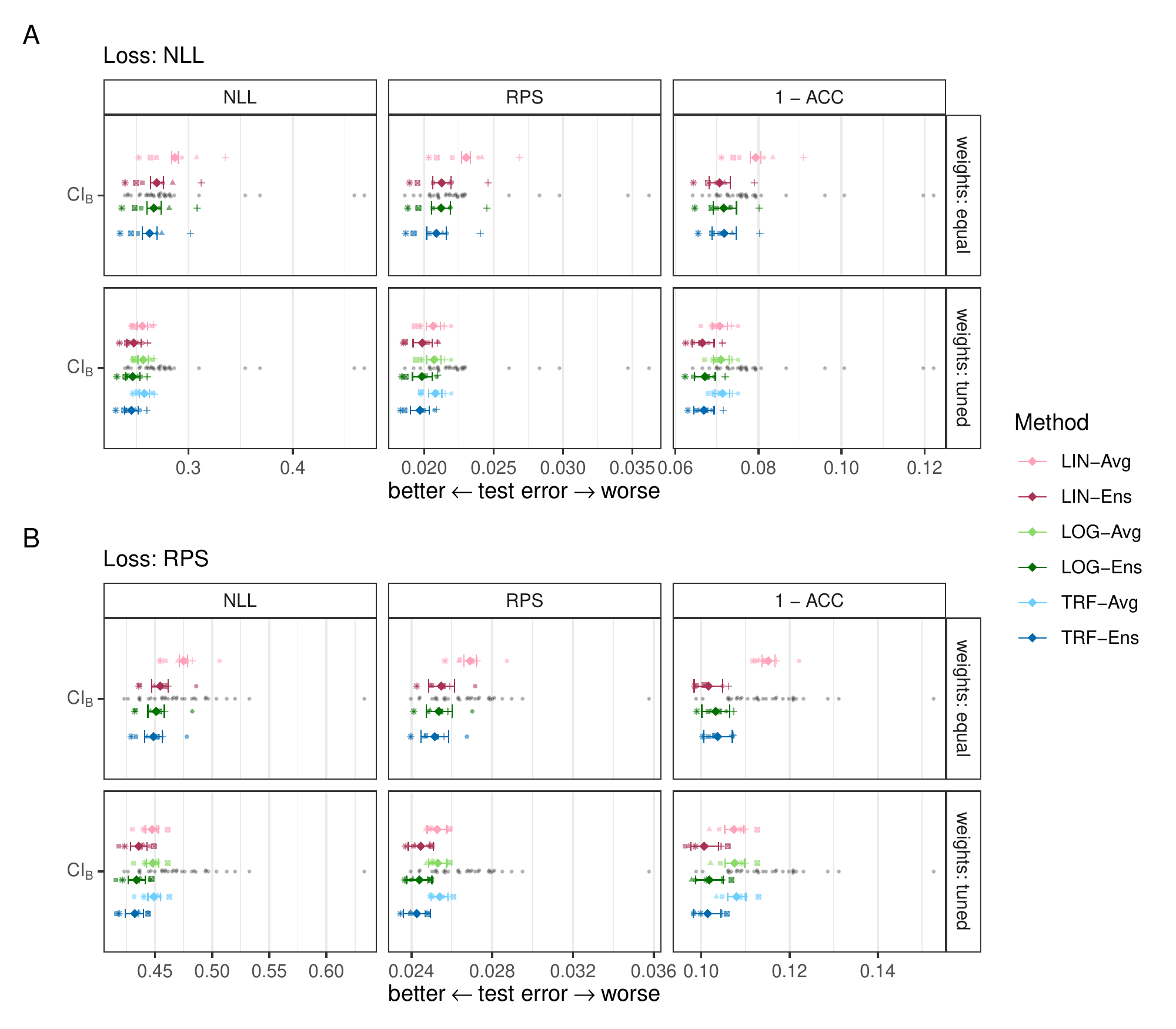}
    \caption{Performance estimates on the MNIST data set. Estimates of individual 
    models are shown as gray dots.
    }
    \label{fig:mnist:indiv}
\end{figure}
\begin{figure}[!ht]
    \centering
    \includegraphics[width=0.7\textwidth]{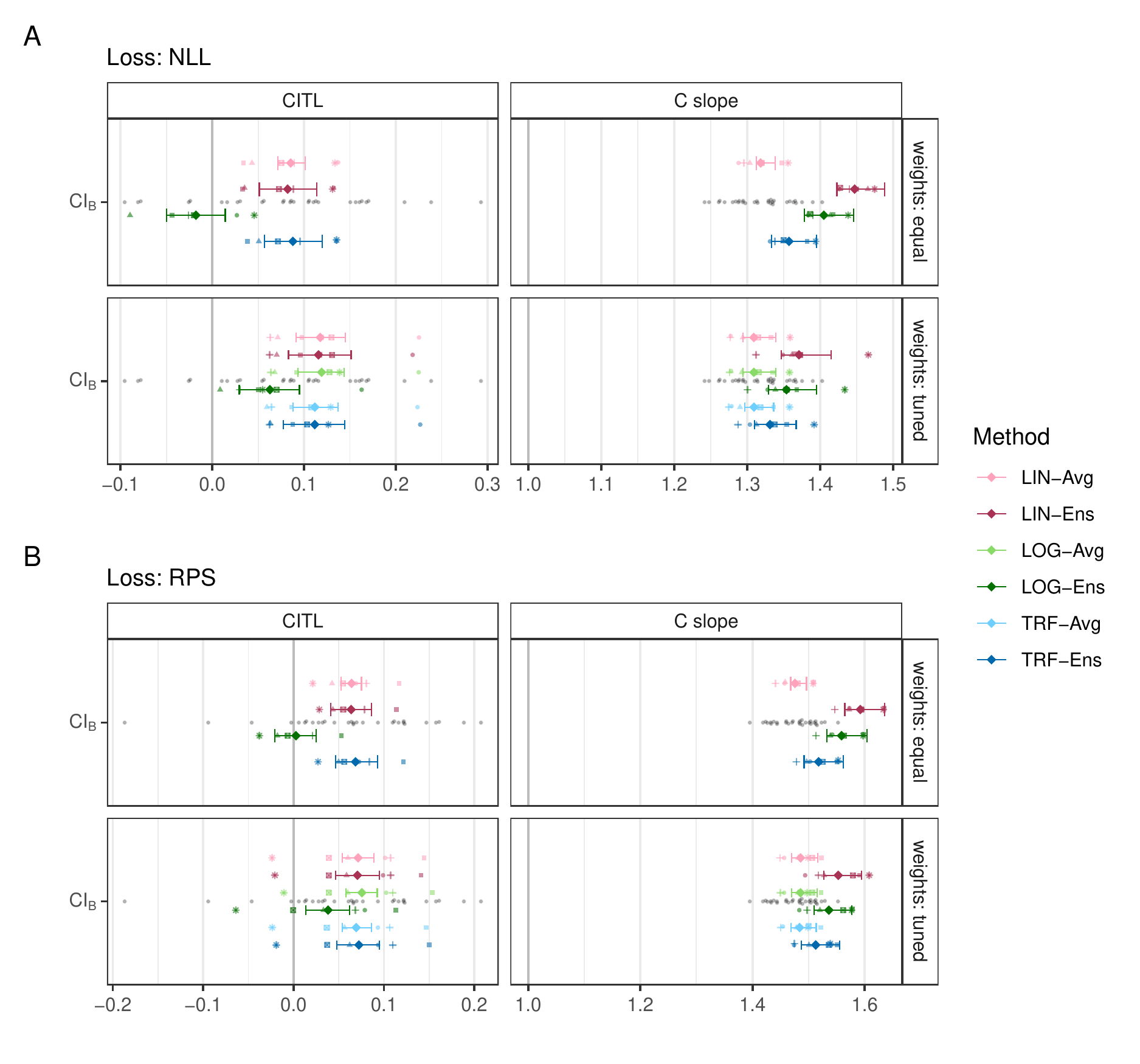}
    \caption{Calibration estimates on the MNIST data set. Calibration-in-the-large (CITL)
    and calibration slope (C slope) of the classical linear (LIN-Ens), classical
    log-linear (LOG-Ens), and transformation (TRF-Ens) ensemble.
    }
    \label{fig:mnist:cal:indiv}
\end{figure}

\clearpage
\subsection{Melanoma}
Fig.~\ref{fig:mela:rel} shows model performance relative to the \silsx{} model.
Conclusions from the main text remain unaltered; even when removing the between-split
variation, variance reduces upon tuning ensemble weights. Performance of individual
models is again shown in Fig.~\ref{fig:mela:indiv}, and calibration metrics in
Fig.~\ref{fig:mela:cal:indiv}. Conclusions from the main text remain the same.
However, how the extreme miscalibration of some members is mitigated by ensembling
can now be seen more clearly.
\begin{figure}[!ht]
    \centering
    \includegraphics[width=0.97\textwidth]{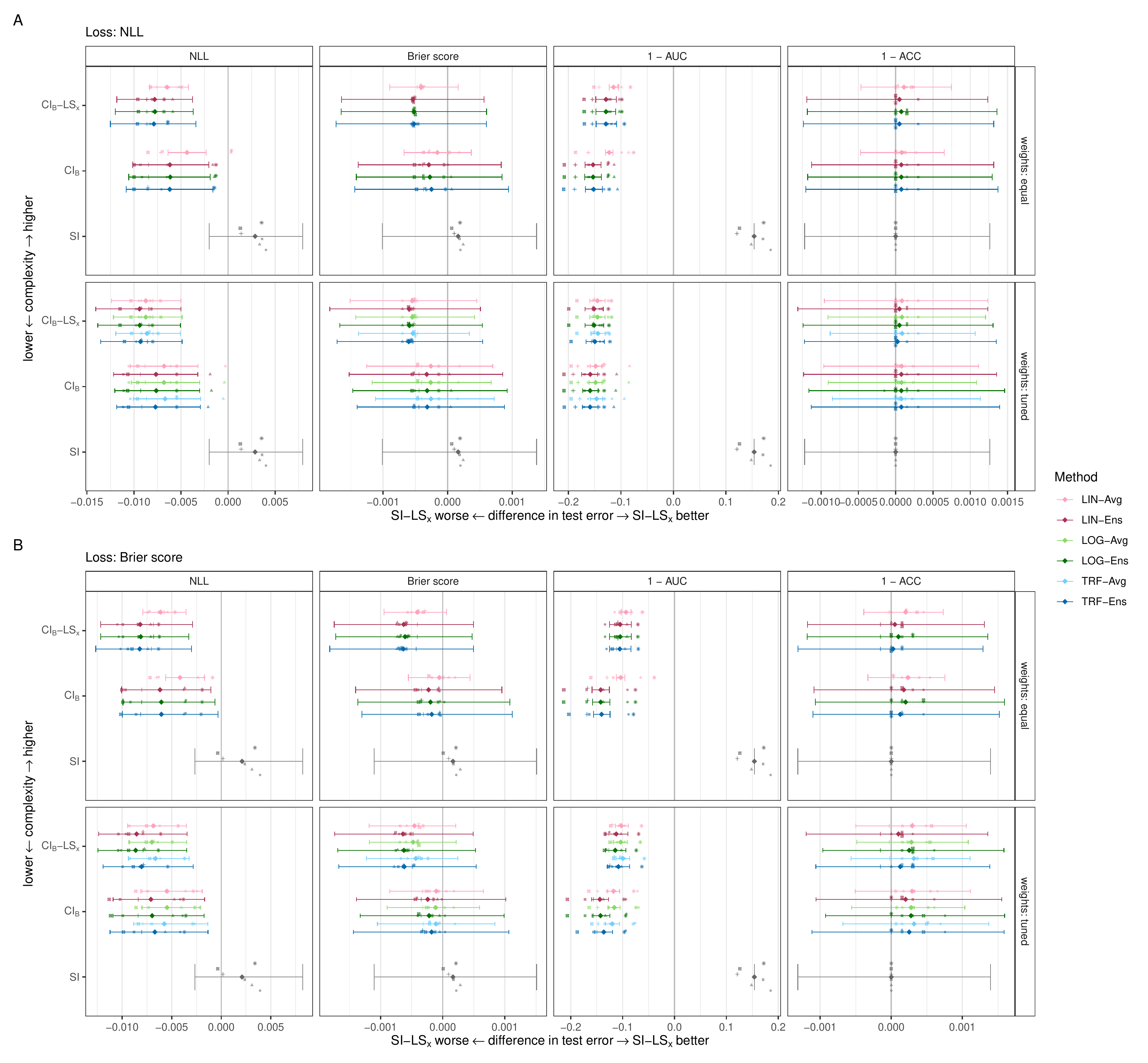}
    \caption{Difference in test error of all models fitted on the melanoma 
    data to the test error of the benchmark \silsx{} model.
    }
    \label{fig:mela:rel}
\end{figure}
\begin{figure}[!ht]
    \centering
    \includegraphics[width=0.97\textwidth]{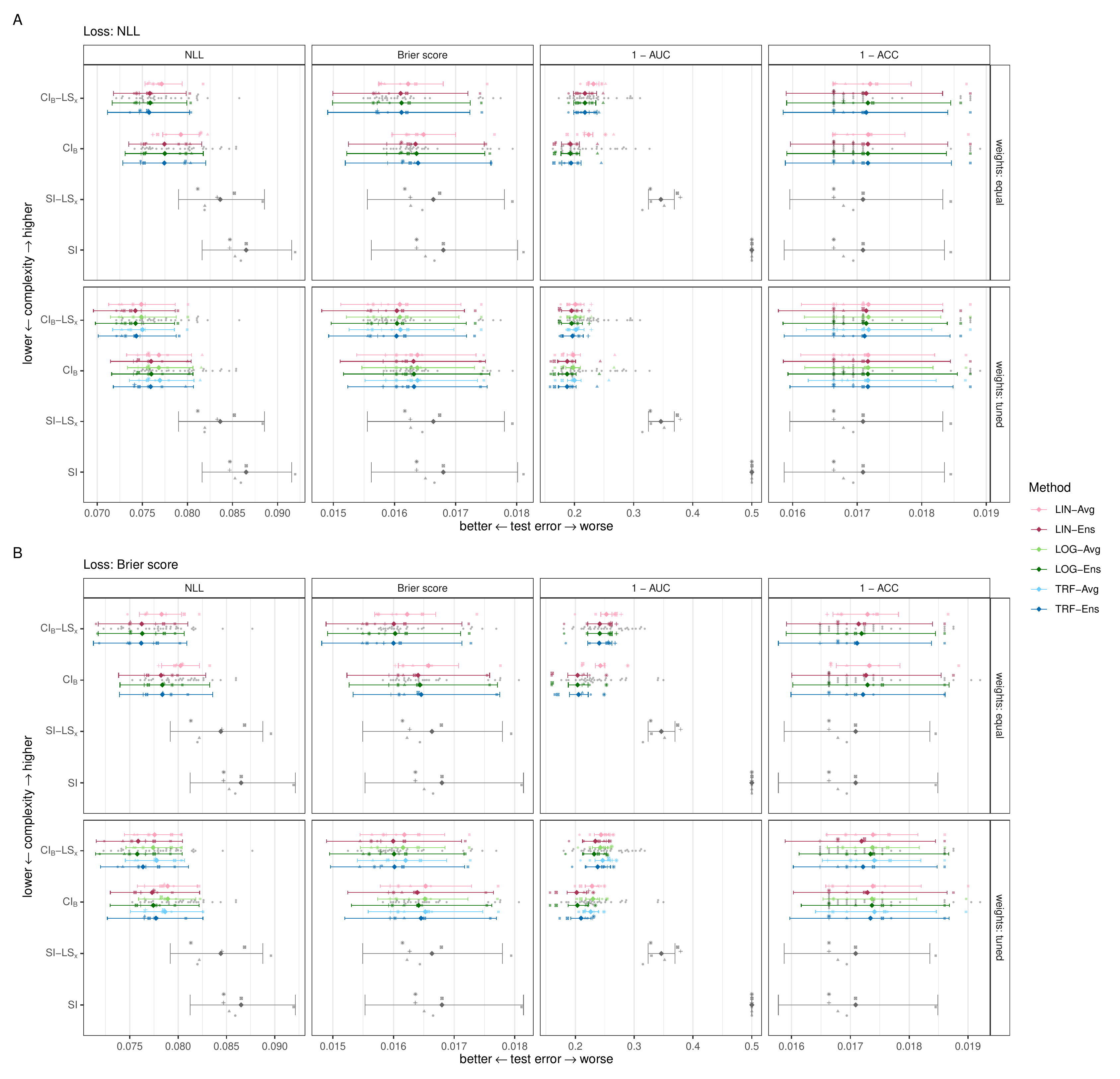}
    \caption{Performance estimates on the melanoma data set. Performance estimates 
    of individual models are shown as gray dots.
    }
    \label{fig:mela:indiv}
\end{figure}
\begin{figure}[!ht]
    \centering
    \includegraphics[width=0.7\textwidth]{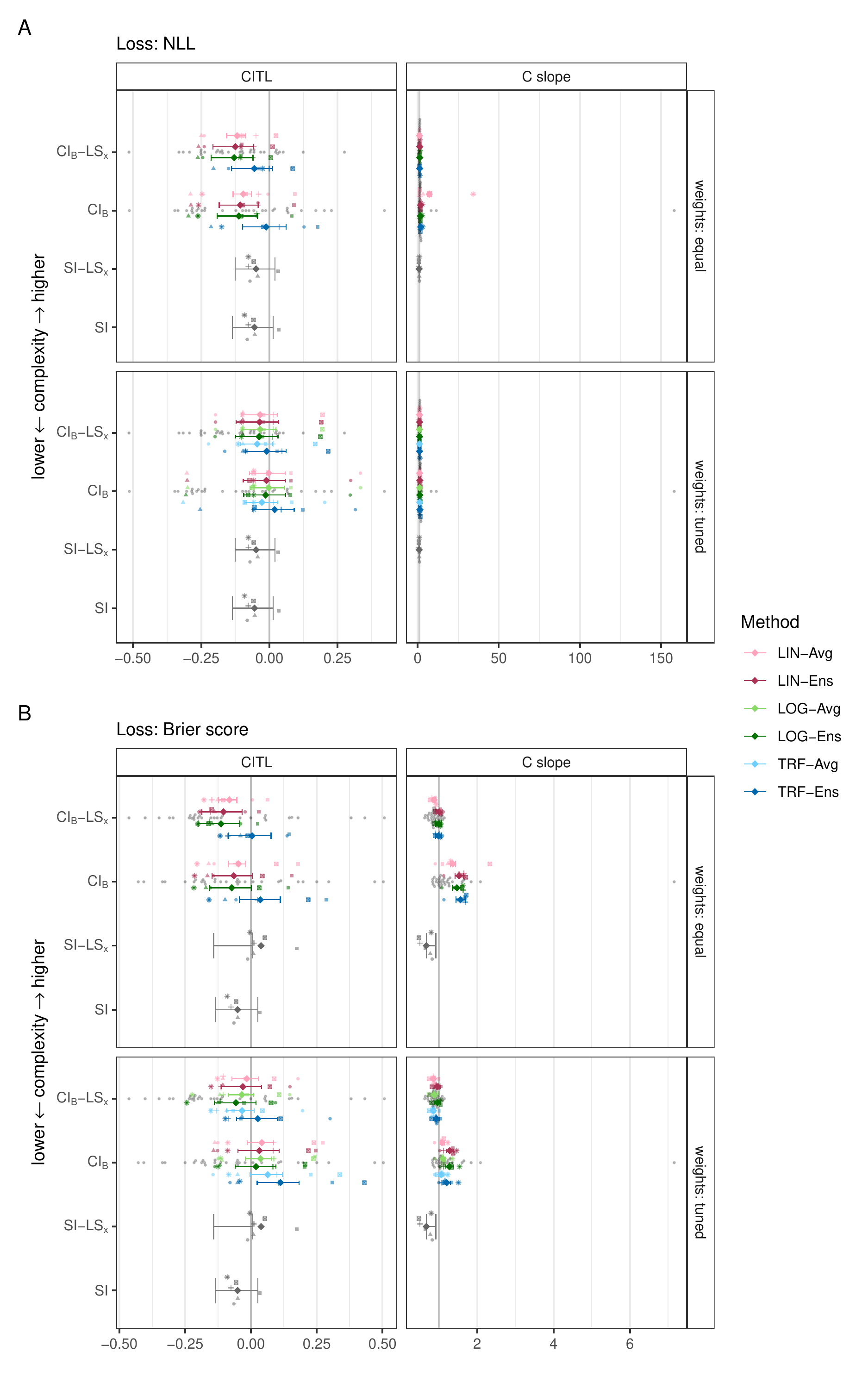}
    \caption{Calibration estimates on the melanoma data set. Calibration-in-the-large (CITL)
    and calibration slope (C slope) of the classical linear (LIN-Ens), classical
    log-linear (LOG-Ens), and transformation (TRF-Ens) ensemble. Estimates of individual
    models are shown as gray dots.
    }
    \label{fig:mela:cal:indiv}
\end{figure}

\clearpage
\subsection{UTKFace}
Fig.~\ref{fig:utkface:rel} shows model performance relative to the \silsx{} model.
Conclusions from the main text remain unaltered; even when removing the between-split
variation, variance reduces slightly upon tuning ensemble weights. Performance of 
individual models is again shown in Fig.~\ref{fig:utkface:indiv}, and calibration metrics 
in Fig.~\ref{fig:utkface:cal:indiv}.
Ensembles of the \cib{} model mitigate stark prediction errors in the individual models.
Calibration is better for the \cib{} and \ciblsx{} models, the classical log-linear 
ensemble yields the most well-calibrated predictions for those.
\begin{figure}[!ht]
    \centering
    \includegraphics[width=0.85\textwidth]{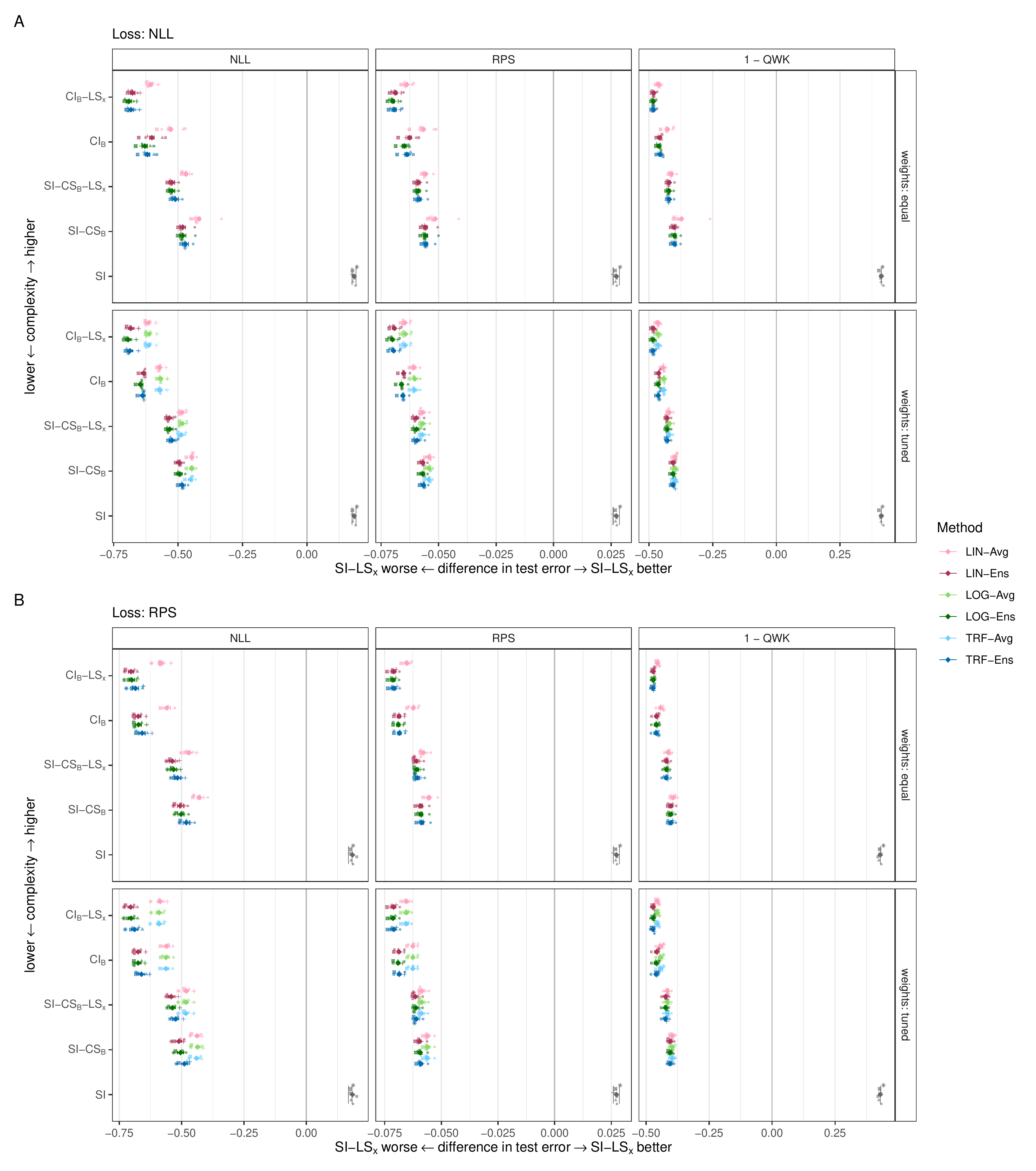}
    \caption{Difference in test error of all models fitted on the UTKFace 
    data to the test error of the benchmark \silsx{} model.
    }
    \label{fig:utkface:rel}
\end{figure}
\begin{figure}[!ht]
    \centering
    \includegraphics[width=0.9\textwidth]{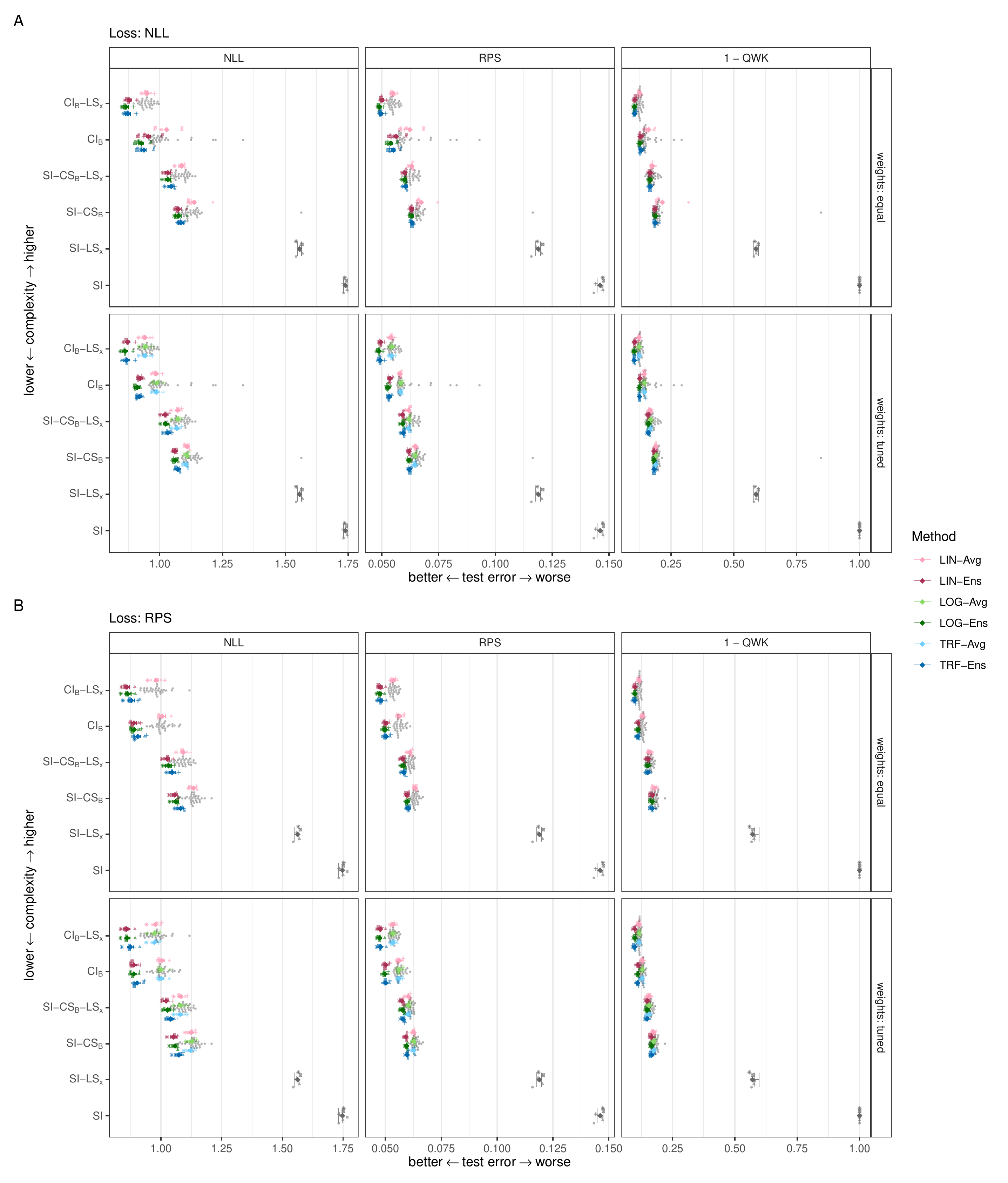}
    \caption{Performance estimates on the UTKFace data set. Estimates of
    individual models are shown as gray dots.
    }
    \label{fig:utkface:indiv}
\end{figure}
\begin{figure}[!ht]
    \centering
    \includegraphics[width=0.7\textwidth]{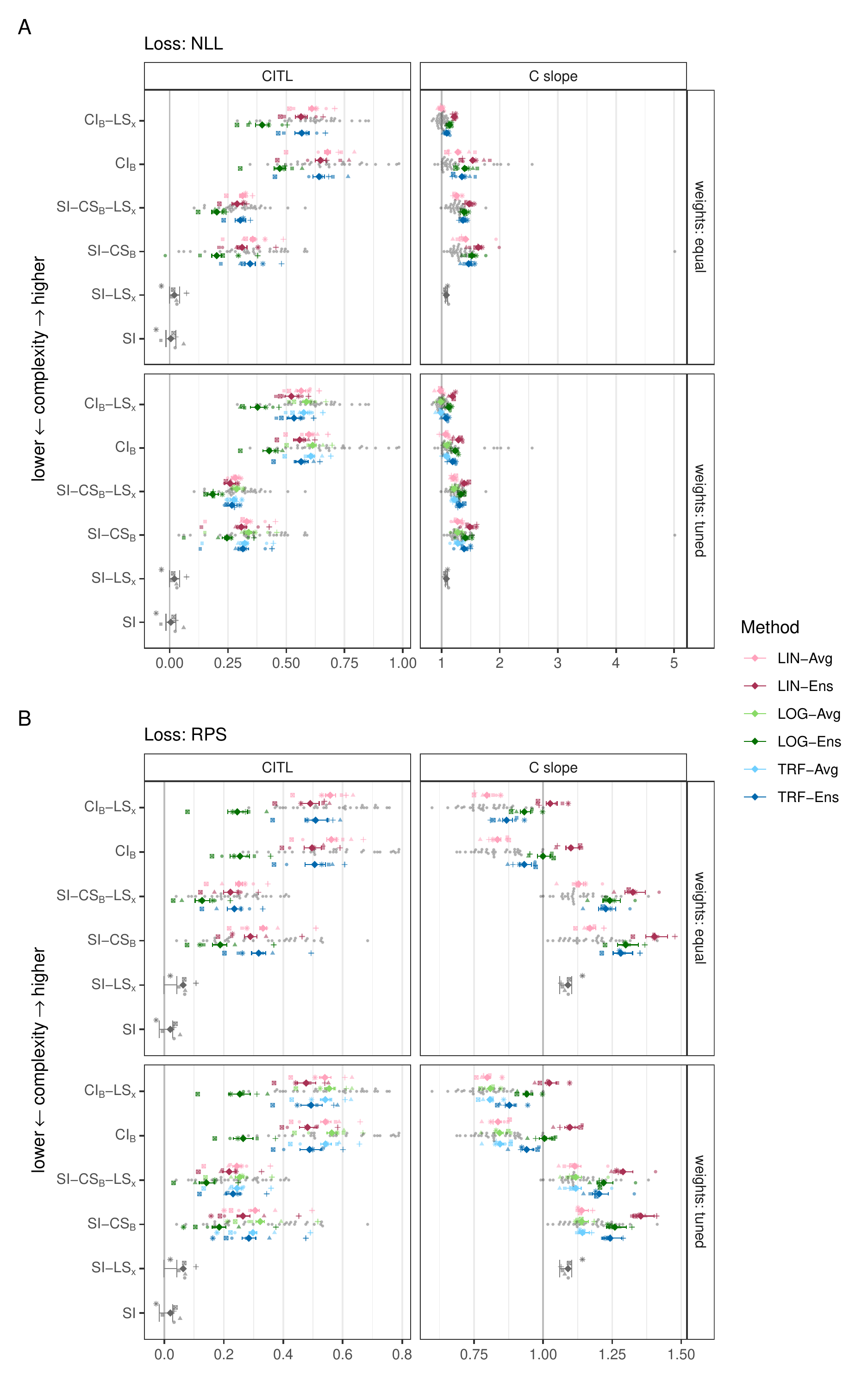}
    \caption{Calibration estimates on the UTKFace data set. Calibration-in-the-large (CITL)
    and calibration slope (C slope) of the classical linear (LIN-Ens), classical
    log-linear (LOG-Ens), and transformation (TRF-Ens) ensemble. Metrics of individual
    models are shown as gray dots.
    }
    \label{fig:utkface:cal:indiv}
\end{figure}
\begin{figure}[!ht]
    \centering
    \includegraphics[width=0.75\textwidth]{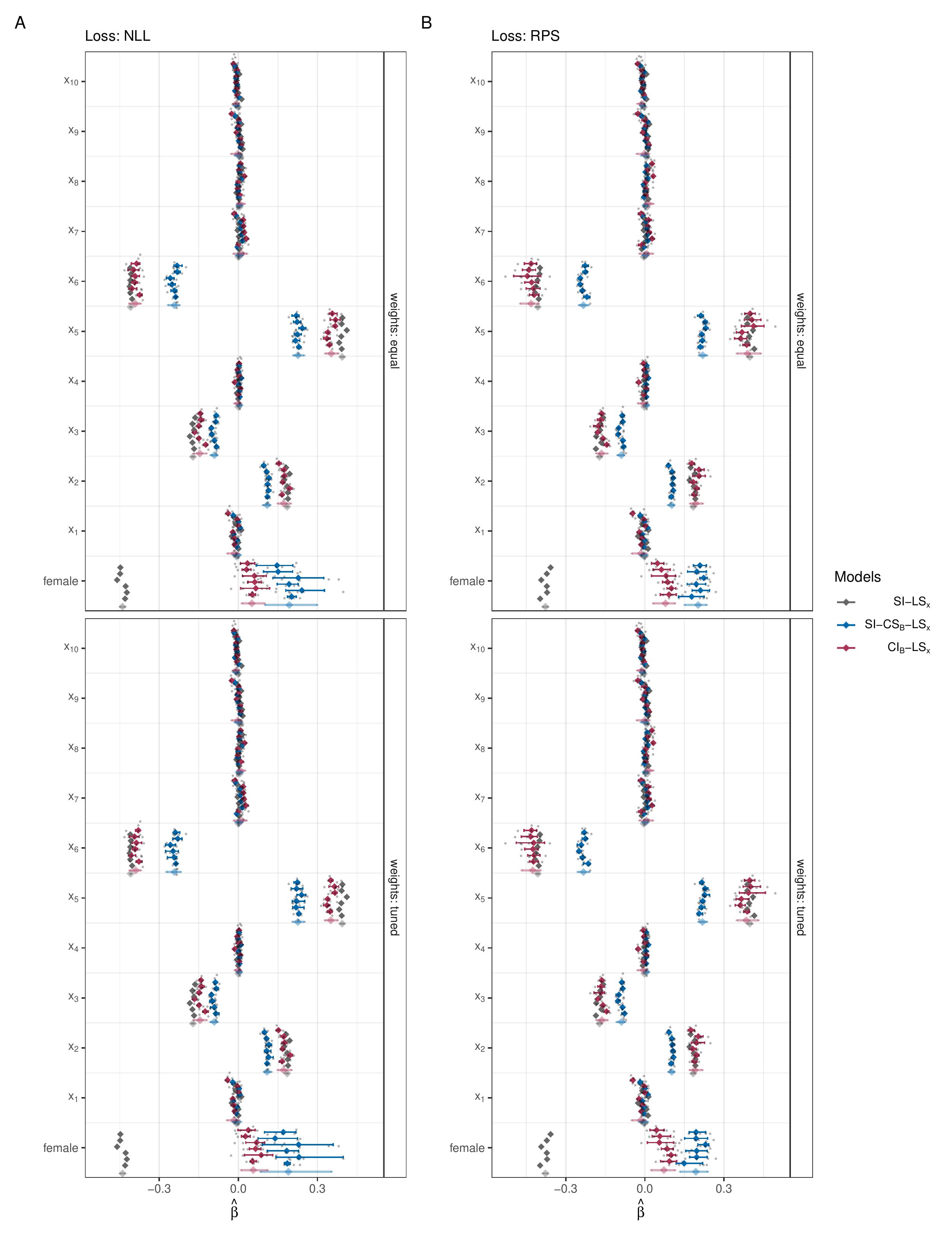}
    \caption{Coefficient estimates on the UTKFace data set. Average log 
    odds-ratios ($\hat\shiftparm$) and 95\% bootstrap confidence 
    intervals of belonging to a higher age class when increasing one of the 10 
    simulated predictors by one unit or switching from male to female for each 
    random split. The average log odds-ratio across splits is shown as transparent 
    dot for each model. Individual odds ratios of all $5 \times 6$ ensemble members or 
    six splits for the \silsx{} model, respectively, are shown as gray dots.
    In the upper panels ensemble members were equally weighted 
    for constructing the ensemble and in the lower panels weights were tuned
    to minimize validation loss. 
    Models were trained by minimizing NLL (\textsf{A}) or RPS (\textsf{B}).
    }
    \label{fig:utkface:or}
\end{figure}

\end{document}